\documentclass[11pt,final]{article}
\usepackage{fullpage}
\usepackage[numbers,compress]{natbib}

\usepackage{microtype}
\usepackage{graphicx}
\usepackage{subcaption}
\usepackage{amssymb}
\usepackage{booktabs} 
\usepackage{bm}
\usepackage{hyperref}
\usepackage{float}
\usepackage{placeins}
\usepackage{stfloats}
\usepackage{pifont}
\usepackage{hyperref}
\usepackage{booktabs}  % For \toprule, \midrule, \bottomrule, \cmidrule
\usepackage{siunitx}
\usepackage{amsmath}
\usepackage{float}
\DeclareMathOperator*{\argmax}{arg\,max}

\usepackage{hyperref}       % hyperlinks
\hypersetup{
     colorlinks   = true,
	 citecolor = teal,
}

\usepackage{amsmath}
\usepackage{amssymb}
\usepackage{mathtools}
\usepackage{amsthm}
\usepackage{mathtools}
\usepackage{multirow}
\usepackage{longtable}
\usepackage{booktabs} % Ensure you have this for \toprule, \midrule, etc.
\DeclarePairedDelimiter{\abs}{\lvert}{\rvert}
\usepackage[capitalize,noabbrev]{cleveref}

\theoremstyle{plain}
\newtheorem{theorem}{Theorem}[section]
\newtheorem{proposition}[theorem]{Proposition}

\theoremstyle{definition}

\theoremstyle{remark}
\newtheorem{remark}[theorem]{Remark}

\usepackage{enumitem}

\usepackage{tikz}
\usepackage{fontawesome5}
\usetikzlibrary{fit} 
\usetikzlibrary{positioning, shapes, arrows.meta, calc}

\newcount\Includeappendix
\newcommand{\apxref}[1]{%
  \ifnum\Includeappendix=0
    the appendix%
  \else
    {#1}%
  \fi
}

\usepackage[textsize=tiny]{todonotes}

\title{From Competition to Collaboration: Designing Sustainable Mechanisms Between LLMs and Online Forums}

\author{%
Niv Fono%
\thanks{%
    {Technion--Israel Institute of Technology
    (\url{niv.fono@campus.technion.ac.il})}}%
\and
Yftah Ziser%
\thanks{%
    {University of Groningen, NVIDIA Research
    (\url{y.ziser@rug.nl})}}%
\and
Omer Ben{-}Porat%
\thanks{%
    {Technion--Israel Institute of Technology
    (\url{omerbp@technion.ac.il})}}%
}
\date{}

\begin{document}
\maketitle
\begin{abstract}
While Generative AI (GenAI) systems draw users away from (Q\&A) forums, they also depend on the very data those forums produce to improve their performance. Addressing this paradox, we propose a framework of sequential interaction, in which a GenAI system proposes questions to a forum that can publish some of them. Our framework captures several intricacies of such a collaboration, including non-monetary exchanges, asymmetric information, and incentive misalignment. We bring the framework to life through comprehensive, data-driven simulations using real Stack Exchange data and commonly used LLMs. We demonstrate the incentive misalignment empirically, yet show that players can achieve roughly half of the utility in an ideal full-information scenario. Our results highlight the potential for sustainable collaboration that preserves effective knowledge sharing between AI systems and human knowledge platforms.
\end{abstract}
\usetikzlibrary{positioning, shapes, arrows.meta, calc}
\begin{figure*}[t]
    \centering
    \scalebox{0.6}{
    \begin{tikzpicture}[
        font=\sffamily,
        >=Stealth,
        align=center,
        every node/.style={font=\Large},
        box/.style={
            draw, thick, fill=#1!10,
            rounded corners=8pt,
            minimum height=2.5cm,
            inner sep=8pt
        },
        arrow/.style={->, thick},
        dashedarrow/.style={->, thick, dashed}
    ]

    % --- Nodes ---
    \node[box=gray] (batch) {\faList~A Batch of\\Questions $Q_t$};

    \node[box=blue, right=1.5cm of batch, align=left] (genai) {
        \textbf{GenAI}\\
        \textit{ranks questions}\\[3pt]
        Question 1 $\bigstar\bigstar$\\[-2pt]
        Question 2 $\bigstar$\\[-2pt]
        $\vdots$\\[-2pt]
        Question $n$ $\bigstar\bigstar\bigstar\bigstar$
    };

    \node[box=green, right=3.7cm of genai, align=left] (forum) {
        \textbf{Forum}\\
        \textit{Applies a Selection Rule $\mathcal{R}$}\\[3pt]
        Question 1 {\checkmark}\\[-2pt]
        Question 2 {\texttimes}\\[-2pt]
        $\vdots$\\[-2pt]
        Question $M$ {\checkmark}
    };

    \node[box=cyan, right=4.1cm of forum, minimum width=2.7cm] (internet) {
        \textbf{WWW}\\
        \faGlobe
    };

    % --- Main arrows ---
    % Shorter first arrow (almost straight)
    \draw[arrow] (batch.east) -- ++(1.0cm,0) -- (genai.west);

    % Centered and longer middle arrows (less curvy)
    \draw[arrow] (genai.east) -- ++(3.2cm,0) node[above, pos=0.5, xshift=0.3cm, align=center]{send top $M$\\questions \faEnvelope} -- (forum.west);

    \draw[arrow] (forum.east) -- ++(3.7cm,0) node[above, pos=0.5, xshift=0.3cm, align=center]{publish selected\\questions \faEnvelope} -- (internet.west);

    % --- Iteration loop arrows (less curved) ---
    \draw[arrow, bend right=10] (internet.north) to node[above, align=center]{Forum Utility \faBolt} (forum.north);
    \draw[arrow, bend left=12] (internet.south) to node[below, align=center]{GenAI Utility \faBolt} (genai.south east);

    % --- Title (larger and higher, centered over full diagram) ---
    \node[
        font=\bfseries\Huge,
        text=blue!60!black
    ] (title) at ($(batch.north)!0.5!(internet.north)+(0,2.0)$) {Iteratively Do \faRedo};

    \end{tikzpicture}
    }
    \caption{Iterative interaction between a GenAI provider and an online Q\&A forum. In each round, the GenAI ranks user-generated questions by their expected model-learning value and submits the top $M$ to the forum. The forum then applies its own selection rule $\mathcal{R}$, publishing only those questions that align with its community objectives. Feedback from the published posts informs both sides in subsequent rounds.}
    \label{fig:teaser}
\end{figure*}

\section{Introduction}
Online Q\&A forums are witnessing a sharp decline in user participation, as people post fewer questions and answers than before. Stack Overflow, for example, experienced a 25\% drop in posts within months of ChatGPT's release-a striking indicator of this trend \cite{10.1093/pnasnexus/pgae400}. Instead of contributing to community-driven platforms like Stack Overflow or Mathematics Stack Exchange, users increasingly turn to generative AI systems for instant answers \cite{impacts}. While this shift offers individual convenience, it carries systemic risks: Q\&A forums generate much of the high-quality, human-produced knowledge that large language models (LLMs) depend on for training 
\cite{almazrouei2023falconseriesopenlanguage,hou2024largelanguagemodelssoftware}, evaluation \cite{dai2024deepseekmoeultimateexpertspecialization}, and benchmarking \cite{lu2024deepseekvlrealworldvisionlanguageunderstanding}. This dependence is especially consequential because manual annotation remains a prohibitively expensive bottleneck in machine learning, making organically generated, high-quality human data sources particularly valuable \cite{10068925,6889457,golazizian2024costefficientsubjectivetaskannotation}. A sustained decline in forum activity may therefore weaken the very epistemic foundations of GenAI, creating a negative feedback loop where AI systems erode the human knowledge sources they rely on. Ultimately, this erosion could not only risk AI progress but also threaten the survival of these online communities, depriving users of open, collaborative spaces for trustworthy knowledge exchange.

Existing responses to this sustainability challenge
suggest restricting data access by GenAI companies~\cite{AKAICHI2025100698} or offer a compensation~\cite{FELGENHAUER2024112017}. Such an 
approach frames GenAI firms and online knowledge communities as adversarial stakeholders in competition over data ownership or user attention, rather than as co-dependent entities within a shared knowledge base.
Recent works have begun to recognize the structural nature of this interdependence. For example,~\citet{taitler2024braesssparadoxgenerativeai} demonstrate that uncoordinated adoption of GenAI technologies can lead to a Braess-like paradox in which both AI development and community vitality suffer in the long run due to declining user participation.
However, even such analyses typically stop short of modeling online forums as fully strategic agents with their own utility functions, governance norms, and content curation policies.

To address this challenge, we propose a form of collaboration that protects the interests of both parties without involving monetary exchange. Specifically, we consider an exchange of goods: GenAI systems could submit queries that they fail to resolve to Q\&A forums, where human experts can provide high-quality answers. This exchange benefits both sides: GenAI companies gain access to expert-level supervision data to improve their models, and forums regain traffic and engagement through challenging, high-value questions.
Recent research analyzing Stack Overflow data in the post-ChatGPT period \cite{helic2025stack} shows that while overall user contributions have declined, the complexity and difficulty of questions have increased, indicating that users are already turning to forums for queries that LLMs cannot handle effectively. Because GenAI firms are unlikely to disclose their internal strategies and forums may reveal only partial information about their curation or moderation policies, such collaboration must be asymmetric by design.
We therefore introduce a novel framework that enables both sides to preserve confidentiality while still engaging in a mutually beneficial interaction.

Figure~\ref{fig:teaser} illustrates our model. We formalize the interaction as a strategic iterative collaboration in which the GenAI company proposes candidate questions for publication and forums decide whether to accept or reject them.
Each actor optimizes its own objective: GenAI companies seek questions that maximize their model improvement potential, while forums prioritize those that enhance community utility.
Importantly, our framework is \textbf{agnostic} to the specific definitions of these utilities, allowing for flexible and evolving notions of value.

We evaluate our framework through large-scale, data-driven simulations grounded in real Stack Exchange data and multiple open-source LLMs. Our simulations reveal that the epistemic value of forum data for GenAI systems is systematically misaligned with community engagement. Questions that are most beneficial for model learning often differ from those that attract high participation or satisfaction within the forums. Nevertheless, we show that each actor can recover a substantial fraction of its utility relative to an idealized full-information collaboration.

\subsection{Our Contribution}
Our work makes several key contributions toward understanding and modeling the interaction between generative AI systems and Q\&A forums:  
\begin{itemize}[wide]

\item \textbf{Realistic collaboration design.} We survey the key dimensions that shape GenAI-forum collaboration and distill three guiding principles for realistic interaction: (i) prohibiting monetary transfers to preserve community trust and autonomy, (ii) addressing intrinsic incentive misalignment between model improvement potential and community engagement value, and (iii) formalizing the collaboration as a process in which GenAI proposes a set of questions and the forum determines their publication. These principles, discussed in Section~\ref{justifications}, capture the social, economic, and strategic constraints that motivate our later framework.

\item \textbf{Game-theoretic framework.} 
We develop a game-theoretic framework that models GenAI providers and Q\&A forums as strategic, interdependent agents---each optimizing distinct yet interconnected utilities. This formulation reinterprets their interaction as a cooperative rather than adversarial game, encompassing both full-information and asymmetric-information settings and providing a foundation for analyzing stable, mutually beneficial equilibria.

\item \textbf{Data-driven game simulation.} 
We empirically instantiate the proposed framework through large-scale simulations of a realistic GenAI-Forum game, using real data from multiple Stack Exchange communities and open-source LLMs. Our analyses reveal a systematic misalignment between GenAI's improvement potential (perplexity) and forum engagement value (view counts), confirming that their interaction constitutes a genuine strategic game rather than an optimization problem. Extensive data-driven experiments and robust bootstrap analysis show that our learned acceptance-aware mechanism recovers over half of the ideal full-information utility for both agents. Specifically, it achieves 51\%--63\% of GenAI's learning potential and 58\%-70\% of forum engagement---without disclosing private information. These findings demonstrate that adaptive, game-theoretic interaction can produce substantial mutual gains even under asymmetric information.

\end{itemize}

\section{General Collaboration Guidelines} \label{justifications}
Before introducing our formal framework, we first outline several key principles that guide its design.
These properties capture essential aspects of GenAI-Forum collaboration and provide the conceptual foundation on which the model is built. Each of the following subsections focuses on one such consideration and explains how it is reflected in the formal setting.

\subsection{No Monetary Transfers}

Markets are often viewed as efficient mechanisms for resolving conflicts of interest through monetary exchange.
A natural question therefore arises: why not allow GenAI companies to pay online forums to publish their questions?
While this arrangement may seem straightforward in theory, it overlooks deeper structural and strategic challenges inherent to such collaborations.
In practice, introducing payments would distort community incentives and compromise forum autonomy.

We highlight two fundamental reasons:
\begin{enumerate}[wide]
    \item \textbf{Erosion of community trust.}
    Forum users contribute voluntarily, motivated by intrinsic factors such as reputation, reciprocity, and community belonging rather than financial gain\cite{meta-analytic}.
    If forums begin accepting payments from AI companies for publishing or answering GenAI-generated questions, users may perceive that their unpaid contributions are being monetized by others.
    This perception could undermine trust and fairness, prompting users to disengage.    Since community participation relies on gamified reputation systems rather than financial rewards, introducing money risks replacing intrinsic motivation with extrinsic incentives, ultimately reducing engagement~\cite{AI2022103101,doi:10.1177/0007650312472609}.
    
    \item \textbf{Loss of autonomy and value misalignment.}
    Financial dependence on GenAI companies would expose forums to legitimacy risks observed in nonprofit--corporate partnerships~\cite{doi:10.1177/0007650312472609,phillips2007tamed}.
    When a community-driven platform becomes financially tied to a powerful corporate actor, its decisions may gradually align with the sponsor's interests rather than its own mission.
\end{enumerate}

In the economic literature, settings with transferable utilities are typically easier to handle than non-transferable utilities (see, e.g., \cite{procaccia2013approximate}).
We do not argue that all GenAI-Forum collaboration should rely on non-transferable utilities, but rather that such settings warrant focused study in this context.

\subsection{Incentive Misalignment}
\label{subsec:incentive_mis}
At first glance, the objectives of GenAI and the Forum appear well-aligned.
Both parties benefit from the publication of well-posed, interesting questions that attract community engagement and yield high-quality answers.
Such questions contribute to the Forum's vitality and simultaneously provide GenAI with valuable external feedback that can inform model improvement.
However, a closer examination suggests that the two sides evaluate their value along different dimensions.
For the Forum, the primary objective is to sustain community participation and content quality---favoring questions that are clear, accessible, and likely to elicit thoughtful human responses.
For GenAI, by contrast, value arises from informational gain: questions that expose uncertainty or challenge the model's current knowledge state are most useful, even if they are niche, ambiguous, or attract limited human engagement.
As a result, a question highly valuable to GenAI may not resonate with the community, while the questions most engaging to users may provide little new information to the model.
This divergence creates a subtle yet significant misalignment in incentives (this theoretical discussion is substantiated empirically later in Section~\ref{sec:findings}).

\subsection{Asymmetric Information and Roles}\label{subsec:asy}
In practice, full transparency between GenAI and the Forum is neither feasible nor desirable.
GenAI companies are naturally cautious about revealing their internal processes or decision criteria.
The questions they propose often reflect areas of uncertainty or weakness in the model. Disclosing the entire set of such questions could expose proprietary information or highlight model failures~\cite{carlini2021extracting,kandpal-etal-2024-user,NEURIPS2024_e01519b4}. 
Therefore, GenAI must act strategically when selecting which questions to share.% \niv{There is a long economics/accounting theory tradition explicitly modeling voluntary disclosure where an informed sender may disclose a subset of signals, and receivers have uncertainty about the sender’s information endowment\cite{10.1145/2558397,selective_disclosure,10.1257/aer.104.8.2400}. In many industries, the correct abstraction of strategic communication is “disclose a subset of items” which might even an outperform or complement internal R\&D, especially when the needed expertise is distributed \cite{aquisition,marginality}.}
At the same time, the Forum maintains discretion over what is ultimately published.
Its decisions are guided by internal policies that aim to protect the quality and sustainability of the community. 
The Forum may also limit the volume of AI-submitted content to preserve the platform's authentic character.

This structure creates an inherently asymmetric form of collaboration: GenAI controls what to propose, and the Forum controls what to publish.
Such asymmetry reflects real-world conditions, where cooperation is possible and mutually beneficial, yet bounded by privacy, autonomy, and strategic caution.
We formalize this asymmetric interaction in the next section.

\section{Problem Formulation}\label{Problem Formulation}

We now introduce the formal asymmetric-information framework. We consider a sequential interaction over $T$ discrete rounds between a GenAI provider and a Q\&A forum, which we refer to as Player~G and Player~F, respectively. Each round includes two stages: in the first stage, Player~G proposes candidate questions for publication. In the second stage, Player~F curates which of these are ultimately published. Formally, in every round $t \in [T]$,
\paragraph{Stage I} Player~G is given a candidate pool of questions, which we denote by $Q_t \subset \mathcal Q$,  
where $\mathcal Q$ is the space of all possible questions. In practice, $Q_t$ models the set of candidate questions identified by Player~G as uncertain, error-prone, or worth external feedback. For instance, this pool may be derived from analyzing past user interactions and selecting those with low implicit satisfaction scores, ambiguous answers, or high model uncertainty---cases where the system would most benefit from human clarification. We abstract the way $Q_t$ is constructed, and assume it is given exogenously. 

As discussed in Subsection~\ref{subsec:asy},
Player~G does not wish to share all of its candidate questions with Player~F. Therefore, Player~G selects a subset $A_t \subseteq Q_t$, where we assume $\abs{A_t} \leq M \ll |Q_t|$ for some $M \in \mathbb N$, and submits $A_t$ to Player~F.

\paragraph{Stage II}
Player~F observes $A_t$, but does not publish $A_t$ as is. Instead, it uses a \emph{selection rule} $\mathcal{R}$, $\mathcal{R}:2^\mathcal Q \rightarrow 2^\mathcal Q$, and selects a subset of published questions $S_t$ such that $S_t = \mathcal{R}(A_t)$. The rule $\mathcal R$ represents Player~F's internal decision mechanism, reflecting its own objectives and community standards. Here too, we assume Player~F is constrained to pick $|S_t| \leq K$ for some $K \in \mathbb N$ for every round $t$.

\paragraph{Utilities}
Published questions draw the attention of the user community in Player~F's platform, who interact with it and possibly generate feedback. Such interaction and feedback benefit Player~F, who aims to optimize KPIs like user engagement measurements, such as views, upvotes, or answers. We keep the actual objective of Player~F abstract, and let \(u_F(q)\) measure the utility of Player~F from each question $q \in \mathcal Q$.  

Similarly, Player~G also derives utility from publishing questions on Player~F, as community interactions provide valuable external signals. For example, user-generated clarification comments, alternative answers, or follow-up discussions can serve as high-quality supervision signals that reveal model weaknesses or knowledge gaps. These signals can be incorporated into Player~G's training or evaluation pipelines to improve its accuracy, calibration, and domain coverage. Here too, we abstract away the richness of the feedback and quantify the utility of a (published) question $q\in \mathcal Q$ by $u_G(q)$.
We assume the utilities are deterministic and are privately known to the respective players. The assumption of deterministic utilities simplifies the analysis by removing uncertainty from the utility space, allowing us to focus on the strategic interaction itself rather than stochastic variability. 

We extend the utilities naturally to sets such that $U_G(S) = \sum_{q \in S} u_G(q)$ for every $S \subseteq \mathcal{Q}$, and similarly $U_F(S) = \sum_{q \in S} u_F(q)$ for Player~F.\footnote{We consider linear utilities for simplicity and practical justification; we elaborate more in \apxref{Section \ref{sec:limitations}}.}
Since the question sets $(S_t)_{t=1}^T$ are induced by the strategies (as $S_t=\mathcal R(A_t)$), we conveniently write the \emph{cumulative utilities} as functions of the strategies, namely, for Player~G,
\begin{align*}
& U_G(\mathbf{A}, \mathcal R)= \sum_{t\in[T]}U_G(\mathcal R(A_t)) = \sum_{t\in[T]}\sum_{q\in S_t} u_G(q),
\end{align*}
where $\mathbf{A}=(A_1,\dots A_T)$. We similarly use $U_F(\mathbf{A}, \mathcal R)$ to denote Player~F's utility.

\begin{remark}
Note that Player~F can set $\mathcal{R}$ to ensure that the proposed question sets are perceived as fair. For instance, by rejecting the entire set of questions if it seems to disproportionately serve Player~G's interests, much like a responder in a dictator game refuses an inequitable offer~\cite{forsythe1994fairness}.
\end{remark}

\subsection{Full-Information and Utility Recovery Rates}\label{subsec:full colab}
The full-information variant is crucial in assessing our framework. In this idealized scenario, both Players~G and~F share complete information about their utilities and candidate questions $Q_t$. A natural cooperative objective is to maximize the Nash product~\cite{bf5a5d3b-3d24-3822-bb09-28237614c31e}, given by 
\begin{align} \label{eq:Nash product}
S_t^\star 
&= \arg\max_{S \subseteq Q_t,\; |S| = K} 
\; U_G(S)\cdot U_F(S).
\end{align}
This formulation captures a balance between the interests of both sides, and could serve as a theoretical upper bound on achievable joint utility. Extending this approach to all rounds $t$, we obtain the \emph{theoretical full-information collaborative} solution $\bm S^\star = (S_1^\star,\dots, S_T^\star)$. Note that solving Equation~\eqref{eq:Nash product} requires complete transparency: both $u_G$ and $u_F$ must be publicly known, and Player~G must disclose its entire candidate pool $Q_t$. These requirements are unrealistic (recall Section~\ref{justifications}), and thus we label this solution theoretical.

To quantify the attractiveness of our framework, we introduce the following notion of \emph{Utility Recovery Rate (URR in short)}. URR measures how much of the full-collaboration utility is recovered by the real-world strategies, and is formally defined by 
\begin{align*}
   URR_G = \frac{U_G(\mathbf{A}, \mathcal R)}{U_G(\bm S^\star)}, \
   URR_F = \frac{U_F(\mathbf{A}, \mathcal R)}{U_F(\bm S^\star)}.
\end{align*}
While URR could theoretically be greater than 1, it is highly unlikely to occur. Values closer to $1$ indicate that the real-world, minimally information-revealing strategy recovers a larger fraction of the utility achievable under full collaboration.

As it turns out, computing URR is impractical, as optimizing the Nash product in Equation~\eqref{eq:Nash product} is computationally intractable. In particular, we show that
\begin{theorem}\label{thm:hard}
The problem in Equation~\eqref{eq:Nash product} is NP-hard. 
\end{theorem}
While maximizing a bilinear function of additive utilities is generally NP-hard (see, e.g., \cite{FainNashIsNP}), we provide a formal proof for our specific setting in the appendix to keep the paper self-contained.

\subsection{Estimating the Utility Recovery Rates}\label{subsec:estimating urr}
Driven by the need to practically assess our framework, we suggest a way to estimate the URR and effectively upper-bound it. Specifically, let $\mathcal H$ be a set of heuristic solutions to the problem in Equation~\eqref{eq:Nash product}, and let $\bm S_h$ be a solution for the entire $T$ round full-information interaction for some $h \in \mathcal H$. We could then take $\tilde U_G = \max_{h \in \mathcal H} U_G(\bm S_h)$ and similarly, $\tilde U_F = \max_{h \in \mathcal H} U_F(\bm S_h)$ to define the \emph{Estimated Utility recovery Rates} (EURR), namely
\begin{align}\label{eq:eurr}
   EURR_G = \frac{U_G(\mathbf{A}, \mathcal R)}{\tilde U_G}, \
   EURR_F = \frac{U_F(\mathbf{A}, \mathcal R)}{\tilde U_F}.
\end{align}
Note that each heuristic in $\mathcal{H}$ is sub-optimal w.r.t. the product in Equation~\eqref{eq:Nash product}, but the variety introduces solutions that are skewed toward one of the players. Thus, since $\argmax_{h \in \mathcal H} U_G(\bm S_h)$ and $\argmax_{h \in \mathcal H} U_F(\bm S_h)$ could be different, taking the best heuristic for each player would typically lead to \emph{over-estimation} of $U_G(\bm S^\star)$ and $U_F(\bm S^\star)$, resulting in \emph{under-estimation} of URR by EURR.\footnote{To illustrate, if the optimal solution yields a product of $9$ by granting each player $3$, and two heuristic solutions yield $4$ by granting one player $4$ and the other player $1$ and vice versa, then EURR is always less than the theoretical URR.}

\subsection{Heuristic Solutions for EURR}\label{subsec:heuristic}
In this subsection, we provide several computationally efficient heuristics $\mathcal H$ to measure the EURRs. For each heuristic, we describe how it picks $S_t$ for every arbitrary round $t$.
\paragraph{Myopic Round Robin (MPP)} Player $G$ adds a previously-unselected question with the highest $u_G$ score. Then, Player~F adds a previously-unselected question with the highest $u_F$ score. This alternating process continues until the question set contains $K$ questions.
\paragraph{Max Sum of Products (MaxSP)} Instead of optimizing the product of sums like in Equation~\eqref{eq:Nash product}, we optimize the sum of products, e.g., solving $
    \max_{S \subset Q_t, \abs{S}=K} \sum_{q\in S} u_G(q) \cdot u_F(q)$.
\paragraph{Greedy Nash Product (GreedyNP)} 
Iteratively selecting the question that yields the largest marginal increase in the Nash product $(U_G \cdot U_F)$ with respect to the previously-obtained set. More formally, for every partial solution $S$, we pick  a question $q$ such that 
    \[
    q \in \argmax_{q\in Q_t \setminus S}\left(
U_G(S\cup \{q \})\right)\cdot \left( U_F(S\cup \{q \})\right),
    \]
    and update $S\gets S\cup \{q \}$. Here too, we stop when the question set contains $K$ questions.

\section{Experimental Setup}
In this section, we present our experimental methodology. Crucially, the modeling and implementation choices below are simplified to obtain a concrete instantiation of the framework; while other choices are possible, these assumptions are sufficient to expose the strategic effects of the framework and are not meant to be exhaustive.

\subsection{Players and Utilities}\label{subsec:heuristics}

\paragraph{Player~G}
To capture GenAI's improvement potential from our framework, we have used several open-source LLMs:
\begin{itemize}[wide]
    \item \textbf{Pythia 6.9B.} EleutherAI-pythia-6.9b \cite{biderman2023pythiasuiteanalyzinglarge}, a fully open \textit{white-box} model with a well-documented training corpus. Its transparency enables us to verify what data it was trained on.
    \item \textbf{Llama 3.1 8B.} \cite{meta_Llama_3_1}, a \textit{black-box} model for which the training data are not publicly disclosed, representing a common closed-source setting.
    \item \textbf{Llama 3.1 8B-Instruct.} \cite{hf_Llama_3_8b_instruct}, an instruction-tuned variant of Llama 3.1 8B, optimized for conversational and reasoning tasks, similarly lacking full disclosure of training sources.
\end{itemize}
This varied set of LLMs allows us to capture a broad range of behaviors and transparency levels, representing both research-oriented and production-like GenAI systems. By comparing open and closed models, we can evaluate whether our framework remains robust under differing information assumptions and access constraints.

%We approximate Player~G's utility using \textbf{perplexity}---a standard measure of model uncertainty---computed over the initial 64 tokens \niv{(common methodological choice in NLP literature \cite{chernyavskiy2021transformerstheendhistory,so2022primersearchingefficienttransformers})} of each question (title plus question content). 

We approximate Player~G's utility using \emph{perplexity}, a standard measure of model uncertainty. We compute it over the initial 64 tokens of each question (title plus question content), following a common methodological choice in the NLP literature \cite{chernyavskiy2021transformerstheendhistory,so2022primersearchingefficienttransformers}.
A higher perplexity value indicates greater model uncertainty \cite{huang2024surveyuncertaintyestimationllms,10.1007/978-3-032-08462-0_25,NEURIPS2024_1bdcb065}, which has been shown to be a strong predictor of prompt failure \cite{gonen2024demystifyingpromptslanguagemodels}, an out-of-distribution (OOD) indicator\cite{ren2023outofdistributiondetectionselectivegeneration} and a key signal for informative data selection in different learning settings \cite{margatina2023activelearningprinciplesincontext}. Consequently, questions with higher perplexity provide stronger learning signals, and thus yield greater utility for Player~G. \footnote{While high perplexity might reflect low-quality or noisy inputs, we assume such cases are relatively rare in our setting given the platform’s active moderation} Finally, we assume $M=100$, i.e., $\abs{A_t}$ is at most 100 questions in every round.

\paragraph{{Player~F}}
On the forum side, Player~F is simulated using real data from Stack Exchange, which
serves as a representative example of large-scale, community-driven Q\&A platforms. We use the
(normalized) \emph{view count} of each question as a proxy for engagement, as it directly captures the
level of human attention and interest that a question attracts~\cite{mustafa2023motivates}. To ensure comparability across interaction rounds and forum domains, view counts are normalized before aggregation. In our
simulations, Player~F is constrained by a moderation-like capacity limit, allowing at most  $K=50$ GenAI-based questions to be published per round (i.e., $|S_t| \leq 50$). This setting mirrors realistic posting and curation dynamics on Stack Exchange, where community attention and moderation resources are inherently limited.
\subsection{Model Selection and Data Integrity}\label{subsec:data int}
To minimize potential data contamination, we focus on Stack Exchange questions posted in July 2024 to July 2025.  This choice ensures that none of the evaluated models could have been trained on the same data. In particular, \textit{Llama 3.1 8B-Instruct}, which is the most recent model in our set, was released at the end of 2024, and therefore could not have been exposed to Stack Exchange content from or after July 2024. By combining models with both transparent (\textit{white-box}) and opaque (\textit{black-box}) training data sources, we obtain a balanced perspective on how our framework performs across different levels of model transparency and accessibility.

We conducted a longitudinal experiment using Stack Exchange data.
Each week---corresponding to one interaction round---we constructed a candidate pool of questions drawn from that period. We then applied the two stages according to players' respective strategies, as we detail shortly in Subsections~\ref{subsec:strategies}.
This process was repeated weekly over a one-year horizon, from July 23, 2024, to July 23, 2025.

To ensure diversity in both linguistic structure and topical expertise, we used data from five distinct Stack Exchange communities, sampled in proportion to the popularity of each domain: \texttt{math} (54,458 samples), \texttt{stackoverflow} (314,949 samples), \texttt{ubuntu} (12,622 samples), \texttt{english} (2,938 samples), and \texttt{latex} (11,441 samples). This cross-domain selection allows our analysis to capture a broad spectrum of question styles, technical subjects, and community behaviors. To obtain a realistic representation of the overall Stack Exchange ecosystem, we combined these datasets into a single corpus, aligning samples across the same date span to reflect genuine temporal activity patterns.
We provide illustrative examples of questions in \apxref{Section~\ref{sec:example questions}}.

\subsection{Player Strategies}\label{subsec:strategies}
In the asymmetric-information setting, Player~G receives a candidate question pool $Q_t$ at each round and extracts from it a subset $A_t$ such that $\abs{A_t}=M=100$. Player~F then applies its selection rule $\mathcal{R}$ to choose the subset of $A_t$ such that the published group of questions $\abs{S_t}\leq K = 50$. Next, we propose several approaches for picking $A_t$ and $\mathcal R$. %

\paragraph{{Player~G}} The strategies we consider for Player~$G$ are as follows:
\begin{itemize}[wide]%[leftmargin=*]
\item \textbf{G-Greedy}
Under this strategy, Player~G selects the 100 questions with the highest perplexity scores. 
\item \textbf{G-Utility Maximization}
Here, Player~G selects the 100 questions that maximize its \emph{expected utility}, 
\[
\mathbb{E}[U(q)] = u_G(q) \cdot \widehat{\pi}(q),
\]
where $\widehat{\pi}(q)$ denotes Player~G's estimate of the probability that a given question will be accepted by Player~F. To construct these estimates, Player~G trains a lightweight classifier that is updated every quarter (i.e., after 13 interaction rounds). Due to limited supervision in early stages, we employ a TF--IDF text representation~\cite{10.1108/eb026526} combined with a Naive Bayes classifier to approximate $\widehat{\pi}(q)$.
During the first quarter, when no acceptance history is available, we assume $\widehat{\pi}(q)=1$ for every $q \in \mathcal Q$.
\item \textbf{Random} Picking $M=100$ questions from $Q_t$ uniformly at random.
\end{itemize}

\paragraph{{Player~F}}
As for the forum, we construct $\mathcal{R}$ as follows. We learn a classifier $\mathcal C = \mathcal C(q)$ from the data to predict whether a question is likely to generate high user engagement, as we elaborate shortly. Then, we pick a threshold $\theta \in \mathbb R$ such that 
\[
S_t = \mathcal R(A_t) = \text{Top}_K \left\{q\in A_t : \mathcal C(q) \geq \theta  \right\},
\]
where $\text{Top}_K\{\cdot\}$ picks the elements with the highest $\mathcal C$ score in the set.

We now describe the way we construct the classifier $\mathcal C$. We train a \textit{BERT}-based classifier \cite{devlin2019bertpretrainingdeepbidirectional} using the Stack Exchange dataset, employing a cross-entropy loss function. During the training process, each question is labeled according to its normalized view count percentile: questions in the top 40\% are labeled as positive samples ($y=1$), those in the bottom 40\% as negative samples ($y=0$), and the middle 20\% are excluded from the training set. As it turns out, the classifier's performance naturally decays over time (due to content shift). To that end, we retrain the classifier every quarter on the cumulative data collected up to that point, where the initial training period covers data from April 23, 2024, to July 22, 2024. A full analysis detailing the classifiers' performance on the weekly datasets is included in \apxref{Section~\ref{app:classifier_performance}}.

Based on validation results, we determine a decision threshold $\theta$ that satisfies the condition $\text{Precision} = 2 \times \text{Recall}$. This dynamic threshold varies within the range of $0.81 \leq \theta \leq 0.9$, yielding a Precision of $\simeq 0.855$ and a Recall of $\simeq 0.43$. This calibration reflects Player~F's preference for \emph{high-quality acceptance} over quantity. In practical terms, Player~F prioritizes Precision by favoring consistently engaging questions---even at the expense of Recall---and is thus willing to publish fewer questions per round.
%The classifier is trained for three epochs on an NVIDIA RTX A6000 GPU, using the cross-entropy loss function. The resulting model achieves an ROC-AUC of 0.74 on the validation set. Based on validation results, we set the decision threshold $\theta$ to $\theta=0.8110$, yielding a precision of 0.7881 and a recall of 0.3940. The threshold $\theta$ is determined to satisfy $\text{Precision} = 2 \times \text{Recall}.$ This calibration reflects Player~F's preference for \emph{high-quality acceptance} over quantity. In practical terms, Player~F prioritizes precision, favoring consistently engaging questions even at the expense of recall, and is therefore willing to publish fewer questions per round.}

\subsection{Robust Analysis and EURR}\label{subsec:full colab strategies}
To compute the EURR (recall Equation~\eqref{eq:eurr}), we estimate $\tilde U_F$ and $\tilde U_G$ via a bootstrap analysis over 50 bootstrap samples. For each bootstrap sample and time $t$, we construct $Q_t$ by randomly sampling half of the questions asked on the platform during that period, yielding a candidate pool ranging from 376 to 5,800 questions per round. We then consider the full-information collaborative approach 
in which the players cooperate perfectly and jointly select the subset $S_t$. In this setting, we evaluate the MPP, MaxSP, and GreedyNP heuristics introduced in Subsection~\ref{subsec:heuristic}, along with random question selection (\textbf{Random}) as a control. We compute the EURR for each bootstrap sample and report the final EURR averaged across all 50 samples. 

% To evaluate the robustness of our approach, we conducted a bootstrap analysis across 50 simulated datasets. For each bootstrap sample and time $t$, we constructed $Q_t$ by randomly sampling half the questions asked in the platform in that period, yielding a candidate pool of questions $Q_t$ ranging between 376 and 5,800 questions per round. We calculated the EURR for each bootstrap sample and obtained the final EURR by averaging across all 50 bootstrap samples.

% To compute the EURR (recall Equation~\eqref{eq:eurr}), we need to estimate $\tilde U_F$ and $\tilde U_G$. To that end, we construct the \emph{idealized experiment} where players engage in perfect cooperation, jointly selecting the subset $S_t$. We employ the MPP, MaxSP, and GreedyNP heuristics mentioned in Subsection~\ref{subsec:heuristic}, along with a random question selection (\textbf{Random}), serving as a control approach.

\subsection{Hardware and Runtime}
All experiments were run on a single NVIDIA RTX A6000 GPU and took roughly 10 hours.

%All experiments were conducted using an NVIDIA RTX A6000 GPU. Each classifier was trained for three epochs, requiring a total training time of 172 minutes and 16 seconds. We recorded the runtime of the strategies and heuristics. For the G-Utility strategy, the total execution times were 2 hours, 43 minutes, and 42 seconds for Pythia 6.9B; 2 hours, 26 minutes, and 36 seconds for Llama 3.1 8B; and 2 hours, 18 minutes, and 23 seconds for Llama 3.1 8B-Instruct.  Under the G-Greedy strategy, Pythia 6.9B completed the task in 53 minutes and 1 second, Llama 3.1 8B in 1 hour, 15 minutes, and 32 seconds, and Llama 3.1 8B-Instruct in 57 minutes and 22 seconds. The random baseline yielded runtimes of 54 minutes and 44 seconds, 1 hour, 13 minutes, and 22 seconds, and 56 minutes and 34 seconds across the three respective models. Finally, the heuristics approach took 2 hours, 25 minutes, and 18 seconds for Pythia 6.9B; 2 hours, 25 minutes, and 12 seconds for Llama 3.1 8B; and 2 hours, 25 minutes, and 19 seconds for Llama 3.1 8B-Instruct.

\section{Findings}\label{sec:findings}
In this section, we present our main findings. Due to space constraints, we defer the exhaustive statistical analysis to \apxref{Subsections~\ref{sec:Full_colab} and \ref{subsec:appndx:asymmetric}}.
\begin{figure*}[t]
    \centering
    
    % --- First Row ---
    \begin{subfigure}[t]{0.32\linewidth}
        \centering
        \includegraphics[width=\linewidth]{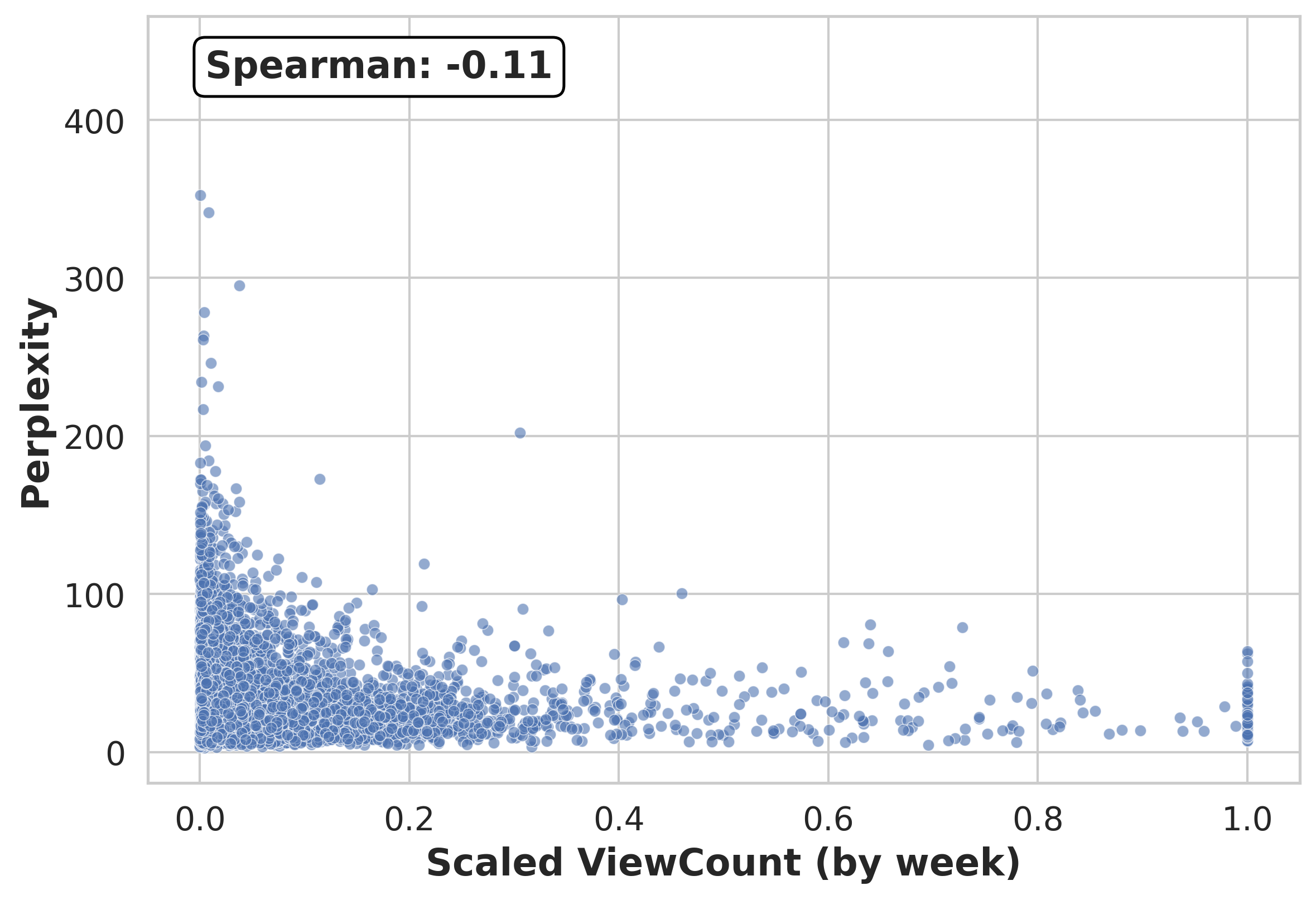}
        \caption{Llama-3.1-8B-Instruct ubuntu}
    \end{subfigure}
    \begin{subfigure}[t]{0.32\linewidth}
        \centering
        \includegraphics[width=\linewidth]{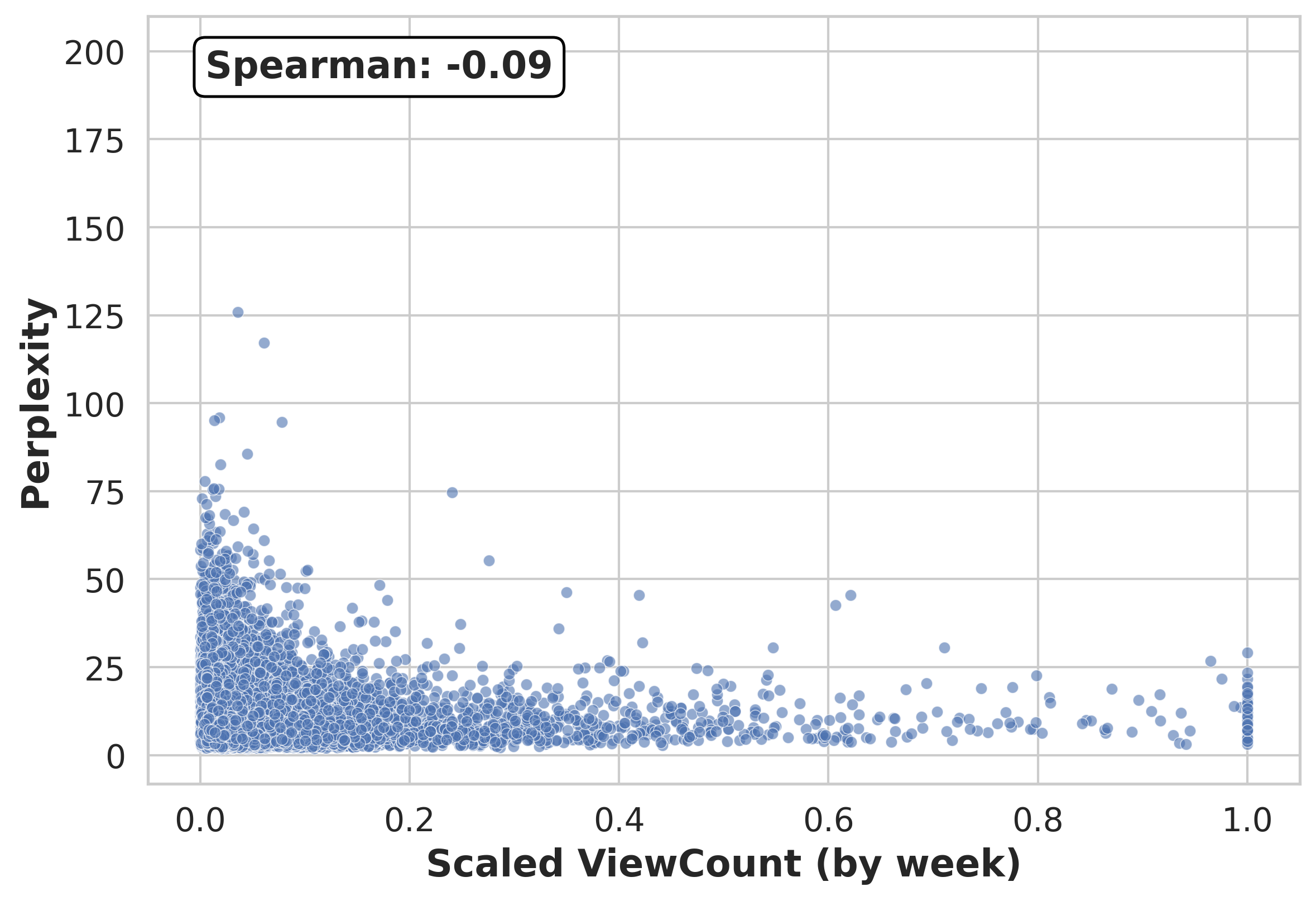}
        \caption{Llama-3.1-8B math}

    \end{subfigure}
    \begin{subfigure}[t]{0.32\linewidth}
        \centering
        \includegraphics[width=\linewidth]{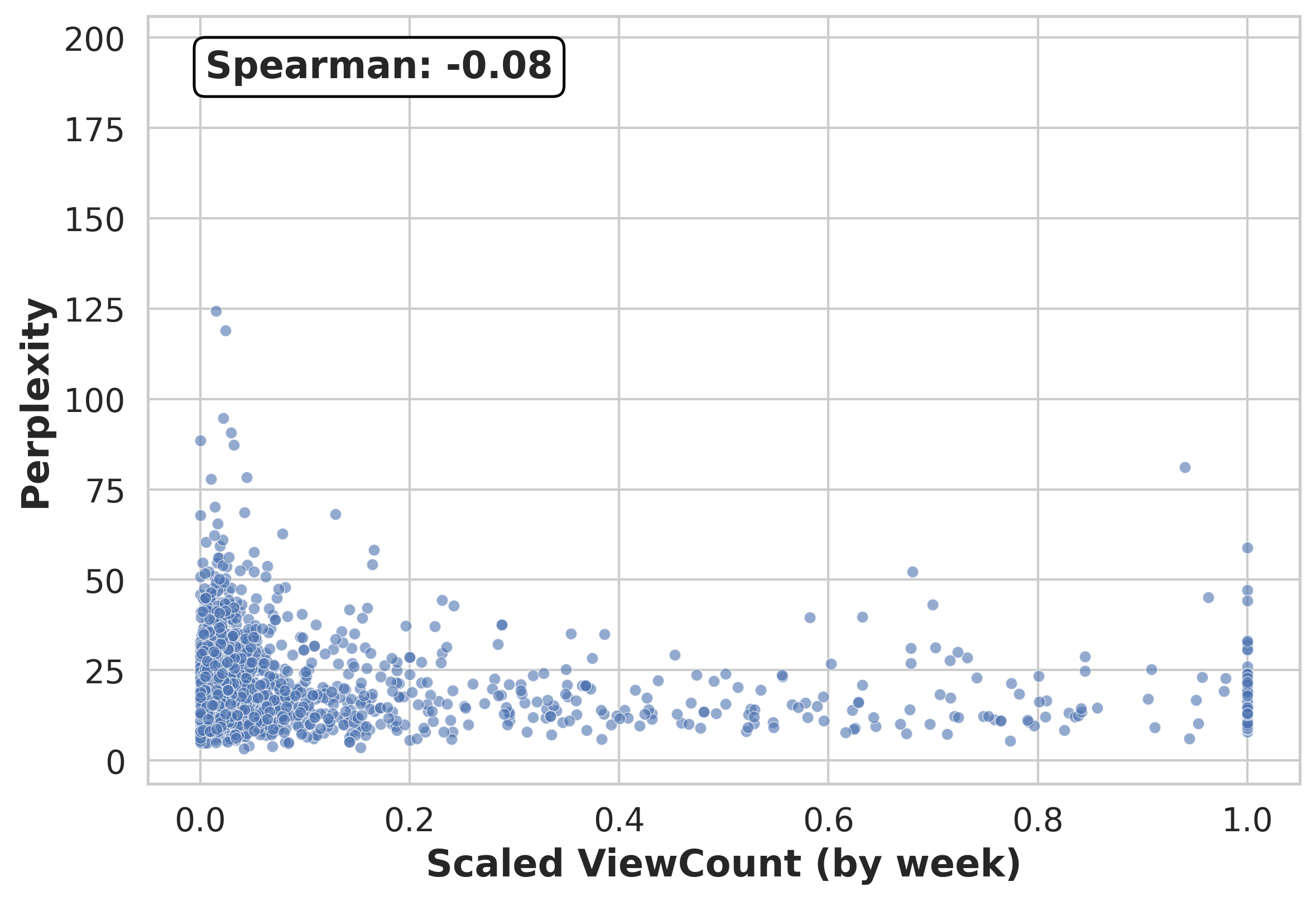}
        \caption{Llama-3.1-8B english}

    \end{subfigure}

    \caption{Relationship between question perplexity and normalized ViewCount across five StackExchange domains. Each plot reports the Spearman correlation coefficient $\rho$. A general pattern of negative correlation emerges, highlighting systematic misalignment between forum engagement and LLM uncertainty.}
    \label{fig:perplexity_vs_viewcount}
\end{figure*}
\begin{table}[t!] 
\small
\centering
\setlength{\tabcolsep}{4pt} % reduce column spacing
\caption{Performance of the idealized full-information and cooperative scenario. Best results are bolded, and * indicates statistical significance based on bootstrapped pairwise comparisons.}
\begin{tabular}{l l c c}
\toprule
\textbf{LLM} & \textbf{Heuristic} & \textbf{Views} & \textbf{Perplexity} \\
\midrule
Pythia 6.9B & MPP & \textbf{2.35}* & 85{,}548.21 \\
Pythia 6.9B & MaxSP & 1.42 & 100{,}934.89 \\
Pythia 6.9B & GreedyNP & 1.41 & \textbf{103{,}705.44}* \\
Pythia 6.9B & Random & 1.05 & 34{,}522.63 \\
\midrule
Llama 3.1 8B  & MPP & \textbf{2.33}* & 79{,}256.36 \\
Llama 3.1 8B  & MaxSP & 1.38 & 90{,}086.83 \\
Llama 3.1 8B  & GreedyNP & 1.40 & \textbf{94{,}678.04}* \\
Llama 3.1 8B  & Random & 1.06 & 32{,}038.47 \\
\midrule
Llama 3.1 8B-Instruct & MPP & \textbf{2.31}* & 131{,}154.55 \\
Llama 3.1 8B-Instruct & MaxSP & 1.23 & 150{,}834.31 \\
Llama 3.1 8B-Instruct & GreedyNP & 1.25 & \textbf{159{,}406.38}* \\
Llama 3.1 8B-Instruct & Random & 1.07 & 48{,}876.31 \\
\bottomrule
\end{tabular}
\label{tab:collab-results}
\end{table}

\subsection{Incentive Misalignment}\label{subsec:findings mis}

We observe a consistently weak negative relationship across all evaluated LLMs and Q\&A domains between question perplexity and view count (mean Spearman correlation $= -0.064$, std $= 0.039$). Figure~\ref{fig:perplexity_vs_viewcount} examines scatter plots for three representative GenAI-forum pairs. As expected, most questions exhibit low view counts, suggesting they are of limited value to the Q\&A platforms. No clear trend emerges as we move along the view-count axis.

The near-zero correlation indicates that the objectives of GenAI models and Q\&A platforms are largely misaligned. If the correlation was close to~1, GenAI systems would naturally select and forward to forums the most engaging and valuable questions for human experts to answer. Instead, this weak association suggests that their interaction forms a genuinely non-trivial game. 
\subsection{Sequential Interaction}\label{subsec: findings sequential}
\paragraph{Full information}
Table ~\ref{tab:collab-results} aggregates the results for the full information setup. For each LLM (Subsection~\ref{subsec:heuristics}) and each heuristic (Subsection~\ref{subsec:heuristic}), we report the cumulative utility obtained over all rounds.
Interestingly, different heuristics favor different players. Specifically, \textbf{GreedyNP}  maximizes Player~G's cumulative perplexity consistently under all LLM choices. In contrast, \textbf{MPP} yields the highest cumulative view counts, favoring Player~F's engagement-oriented objectives across all LLMs.\textbf{Random} performs substantially worse across all metrics.

These findings highlight that even under conditions of complete transparency and alignment, inherent trade-offs between model improvement potential (Player~G) and user engagement (Player~F) remain unavoidable. Moving forward, we use $\tilde U_G$ as the highest utility obtained for Player~G,and similarly  $\tilde U_F$ for Player~F. We remind the reader that this choice overestimates $U_G(\bm S^\star)$ and $U_F(\bm S^\star)$, as Subsection~\ref{subsec:estimating urr} describes.

\paragraph{Asymmetric information}

\begin{table}[htbp]
\centering
\caption{ Performance under asymmetric information. The  mean of the normalized view counts and cumulative
perplexity represent Player~F and Player~G’s utilities, respectively. The EURR is computed against
the worst-case cooperative solution (recall Equation \eqref{eq:eurr}).}
\label{tab:combined_performance}
\begin{tabular}{llcccc}
\hline
\textbf{LLM} & \textbf{Strategy} & \textbf{Views} & \textbf{Perplexity} & \textbf{$EURR_F$} & \textbf{$EURR_G$} \\
\hline
Pythia 6.9B & G-Greedy &0.721 &51729.693 & 0.309 & 0.499 \\
Pythia 6.9B & G-Utility &\textbf{1.621} &\textbf{65346.304} & \textbf{0.696} & \textbf{0.630} \\
Pythia 6.9B & Random &0.867 &16168.148 & 0.371 & 0.156 \\
\hline
Llama 3.1 8B & G-Greedy &0.601 &35985.981 & 0.260 & 0.380 \\
Llama 3.1 8B & G-Utility &\textbf{1.424} & \textbf{48307.862}& \textbf{0.616} & \textbf{0.510} \\
Llama 3.1 8B & Random &0.854 &14044.069 & 0.370 & 0.148 \\
\hline
Llama 3.1 8B-Instruct & G-Greedy &0.568 & 63644.422& 0.248 & 0.399 \\
Llama 3.1 8B-Instruct & G-Utility &\textbf{1.348} &\textbf{83868.269} & \textbf{0.589} & \textbf{0.526} \\
Llama 3.1 8B-Instruct & Random &0.856 &21509.105 & 0.374 & 0.135 \\
\hline
\end{tabular}
\end{table}

Table~\ref{tab:combined_performance} reports player utilities for each of Player~G's strategies. Across all evaluated language models, \textbf{G-Utility} consistently outperforms the \textbf{G-Greedy} strategy. This result is statistically significant under all LLMs. In other words, incorporating acceptance probability estimates yields higher total perplexity than naively maximizing perplexity alone across all LLMs.

For example, with Pythia-6.9b, \textbf{G-Greedy} achieves a cumulative perplexity of 51,729, whereas \textbf{G-Utility} increases it to 65,346, while simultaneously boosting community engagement from 0.721 to 1.621 normalized views. Similar gains are observed for both variants of Llama-3.1, where expected-utility selection yields perplexity improvements of +24\% to +34\%  relative to greedy, alongside substantial increases in views. Interestingly, \textbf{G-Greedy} underperforms relative to the random baseline w.r.t. normalized views, as Player~G is agnostic toward Player~F's utility under both strategies.

Ultimately, Table~\ref{tab:combined_performance} reports the EURR for each player. \textbf{G-Utility} allows for recovering 58\%-70\% of Player~F's utility, and similarly 51\%-63\% for Player~G's utility compared to the idealized full-information, collaborative scenario. We remind the reader that these are \emph{under-estimates}, highlighting that potential collaboration on the basis of our framework could be even more appealing to both parties.

\section{Related Work}

\paragraph{Data Depletion and Feedback Loops.} Users increasingly abandon their roles as knowledge contributors once AI tools provide immediate assistance \cite{10.1093/pnasnexus/pgae400}, creating a ``participation decline'' \cite{scientificreports2024genai}. This trend poses significant risks to the long-term viability of AI development~\cite{10.1145/3706598.3714069}. Both open-source and commercial LLMs critically rely on rich, publicly available datasets \cite{ahmed2024studyingllmperformanceclosed}, such as \textit{The Pile}~\cite{pile} and \textit{CodeInsight} \cite{beau2024codeinsightcurateddatasetpractical}, which incorporate data from Q\&A platforms like Stack Overflow. As these sources experience declining contributions, AI models face a ``data starvation'' risk: reduced human input may degrade model quality, discouraging further participation and creating a self-reinforcing cycle. 

\paragraph{Participation Dynamics and Incentives} Existing approaches to this sustainability challenge have significant limitations. Proposals to restrict community data usage \cite{AKAICHI2025100698} or implement basic incentive mechanisms \cite{FELGENHAUER2024112017} address only surface-level symptoms and neglect the structural dynamics driving participation decline. Some works \cite{zhang2025fairshare} suggest sophisticated data pricing, but as we elaborate in Section~\ref{justifications}, two-sided information non-monetary markets need more careful attention. Behavioral interventions, such as \citeauthor{taitler2025selective}'s ``selective response'' \cite{taitler2025selective}, attempt to steer user behavior but still treat AI providers and knowledge communities as competing stakeholders rather than collaborators.

Finally, our framework is based on cooperative game theory, similarly to many recent works in the machine learning domain \cite{wu2025munba,zeng2024fairness}.

\section{Conclusion}
In this work, we addressed a critical sustainability challenge at the intersection of generative AI development and online knowledge forums. While LLMs rely heavily on high-quality forum data for continual improvement, the rise of these models threatens the very forums that supply such data, creating a dynamic reminiscent of the ``Tragedy of the Commons'' \cite{doi:10.1126/science.162.3859.1243} .

We surveyed the key guidelines and then formalized the problem through a game-theoretic model of asymmetric information. We designed a practical framework where GenAI companies propose questions and forums selectively accept them, assuming non-transferable utilities. Our empirical analysis demonstrates that player utilities are systematically misaligned, suggesting that the interaction between GenAI systems and human-driven forums constitutes a genuine strategic game rather than an optimization problem. Nonetheless, our results show that a lightweight, acceptance-aware strategy for the GenAI company and a simple threshold classifier for the forum can recover a substantial portion of the ideal full-information variant---achieving up to 70\% of forum utility and 63\% of GenAI utility relative to full-information collaboration. 
Overall, our findings demonstrate that sustainable collaboration between GenAI companies and online forums is feasible and beneficial.

We note that our results are derived under a set of modeling assumptions chosen for tractability and clarity. In \apxref{Section~\ref{sec:limitations}}, we elaborate on the resulting limitations and discuss their implications. Future work could extend ours by addressing or relaxing some of these assumptions.

\section*{Acknowledgments}
This research was supported by the Israel Science Foundation (ISF; Grant No. 3079/24). We thank the anonymous reviewers for their valuable feedback.

\bibliographystyle{plainnat}
\bibliography{ref}

@article{10.1093/pnasnexus/pgae400,
    author = {del Rio-Chanona, R Maria and Laurentsyeva, Nadzeya and Wachs, Johannes},
    title = {Large language models reduce public knowledge sharing on online Q\&amp;A platforms},
    journal = {PNAS Nexus},
    volume = {3},
    number = {9},
    pages = {pgae400},
    year = {2024},
    month = {09},

    issn = {2752-6542},
    doi = {10.1093/pnasnexus/pgae400},
    url = {https://doi.org/10.1093/pnasnexus/pgae400},
    eprint = {https://academic.oup.com/pnasnexus/article-pdf/3/9/pgae400/59316621/pgae400.pdf},
}

@inproceedings{NEURIPS2024_e01519b4,
 author = {Maini, Pratyush and Jia, Hengrui and Papernot, Nicolas and Dziedzic, Adam},
 booktitle = {Advances in Neural Information Processing Systems},
 doi = {10.52202/079017-3941},
 editor = {A. Globerson and L. Mackey and D. Belgrave and A. Fan and U. Paquet and J. Tomczak and C. Zhang},
 pages = {124069--124092},
 publisher = {Curran Associates, Inc.},
 title = {LLM Dataset Inference: Did you train on my dataset?},
 url = {https://proceedings.neurips.cc/paper_files/paper/2024/file/e01519b47118e2f51aa643151350c905-Paper-Conference.pdf},
 volume = {37},
 year = {2024}
}

@inproceedings{kandpal-etal-2024-user,
    title = "User Inference Attacks on Large Language Models",
    author = "Kandpal, Nikhil  and
      Pillutla, Krishna  and
      Oprea, Alina  and
      Kairouz, Peter  and
      Choquette-Choo, Christopher A.  and
      Xu, Zheng",
    editor = "Al-Onaizan, Yaser  and
      Bansal, Mohit  and
      Chen, Yun-Nung",
    booktitle = "Proceedings of the 2024 Conference on Empirical Methods in Natural Language Processing",
    month = nov,
    year = "2024",
    address = "Miami, Florida, USA",
    publisher = "Association for Computational Linguistics",
    url = "https://aclanthology.org/2024.emnlp-main.1014/",
    doi = "10.18653/v1/2024.emnlp-main.1014",
    pages = "18238--18265"
}

@inproceedings{10.5555/3709347.3743523,
author = {Amanatidis, Georgios and Birmpas, Georgios and Lazos, Philip and Leonardi, Stefano and Reiffenh\"{a}user, Rebecca},
title = {Algorithmically Fair Maximization of Multiple Submodular Objective Functions},
year = {2025},
isbn = {9798400714269},
publisher = {International Foundation for Autonomous Agents and Multiagent Systems},
address = {Richland, SC},
booktitle = {Proceedings of the 24th International Conference on Autonomous Agents and Multiagent Systems},
pages = {115–123},
numpages = {9},
location = {Detroit, MI, USA},
series = {AAMAS '25}
}

@misc{taitler2024braesssparadoxgenerativeai,
      title={Braess's Paradox of Generative AI}, 
      author={Boaz Taitler and Omer Ben-Porat},
      year={2024},
      eprint={2409.05506},
      archivePrefix={arXiv},
      primaryClass={cs.GT},
      url={https://arxiv.org/abs/2409.05506}, 
}

@article{
doi:10.1126/science.162.3859.1243,
author = {Garrett Hardin },
title = {The Tragedy of the Commons},
journal = {Science},
volume = {162},
number = {3859},
pages = {1243-1248},
year = {1968},
doi = {10.1126/science.162.3859.1243}}

@article{taitler2025selective,
  title={Selective response strategies for genai},
  author={Taitler, Boaz and Ben-Porat, Omer},
  journal={arXiv preprint arXiv:2502.00729},
  year={2025}
}

@article{AKAICHI2025100698,
title = {A comprehensive review of usage control frameworks},
journal = {Computer Science Review},
volume = {56},
pages = {100698},
year = {2025},
issn = {1574-0137},
doi = {https://doi.org/10.1016/j.cosrev.2024.100698},
url = {https://www.sciencedirect.com/science/article/pii/S1574013724000819},
author = {Ines Akaichi and Sabrina Kirrane}
}

@inproceedings{margatina2023activelearningprinciplesincontext,
    title = "Active Learning Principles for In-Context Learning with Large Language Models",
    author = "Margatina, Katerina  and
      Schick, Timo  and
      Aletras, Nikolaos  and
      Dwivedi-Yu, Jane",
    editor = "Bouamor, Houda  and
      Pino, Juan  and
      Bali, Kalika",
    booktitle = "Findings of the Association for Computational Linguistics: EMNLP 2023",
    month = dec,
    year = "2023",
    address = "Singapore",
    publisher = "Association for Computational Linguistics",
    url = "https://aclanthology.org/2023.findings-emnlp.334/",
    doi = "10.18653/v1/2023.findings-emnlp.334",
    pages = "5011--5034",
}

@misc{gonen2024demystifyingpromptslanguagemodels,
      title={Demystifying Prompts in Language Models via Perplexity Estimation}, 
      author={Hila Gonen and Srini Iyer and Terra Blevins and Noah A. Smith and Luke Zettlemoyer},
      year={2024},
      eprint={2212.04037},
      archivePrefix={arXiv},
      primaryClass={cs.CL},
      url={https://arxiv.org/abs/2212.04037}, 
}

@misc{huang2024surveyuncertaintyestimationllms,
      title={A Survey of Uncertainty Estimation in LLMs: Theory Meets Practice}, 
      author={Hsiu-Yuan Huang and Yutong Yang and Zhaoxi Zhang and Sanwoo Lee and Yunfang Wu},
      year={2024},
      eprint={2410.15326},
      archivePrefix={arXiv},
      primaryClass={cs.CL},
      url={https://arxiv.org/abs/2410.15326}, 
}

@inproceedings{NEURIPS2024_1bdcb065,
 author = {Ye, Fanghua and Yang, Mingming and Pang, Jianhui and Wang, Longyue and Wong, Derek F. and Yilmaz, Emine and Shi, Shuming and Tu, Zhaopeng},
 booktitle = {Advances in Neural Information Processing Systems},
 doi = {10.52202/079017-0491},
 editor = {A. Globerson and L. Mackey and D. Belgrave and A. Fan and U. Paquet and J. Tomczak and C. Zhang},
 pages = {15356--15385},
 publisher = {Curran Associates, Inc.},
 title = {Benchmarking LLMs via Uncertainty Quantification},
 url = {https://proceedings.neurips.cc/paper_files/paper/2024/file/1bdcb065d40203a00bd39831153338bb-Paper-Datasets_and_Benchmarks_Track.pdf},
 volume = {37},
 year = {2024}
}

@misc{chernyavskiy2021transformerstheendhistory,
      title={Transformers: "The End of History" for NLP?}, 
      author={Anton Chernyavskiy and Dmitry Ilvovsky and Preslav Nakov},
      year={2021},
      eprint={2105.00813},
      archivePrefix={arXiv},
      primaryClass={cs.CL},
      url={https://arxiv.org/abs/2105.00813}, 
}

@misc{so2022primersearchingefficienttransformers,
      title={Primer: Searching for Efficient Transformers for Language Modeling}, 
      author={David R. So and Wojciech Mańke and Hanxiao Liu and Zihang Dai and Noam Shazeer and Quoc V. Le},
      year={2022},
      eprint={2109.08668},
      archivePrefix={arXiv},
      primaryClass={cs.LG},
      url={https://arxiv.org/abs/2109.08668}, 
}

@InProceedings{10.1007/978-3-032-08462-0_25,
author="Tur{\'o}n, Pablo
and Cuadros, Montse",
editor="Corchado, Emilio
and Quinti{\'a}n, H{\'e}ctor
and Troncoso Lora, Alicia
and P{\'e}rez Garc{\'i}a, Hilde
and Jove P{\'e}rez, Esteban
and Calvo Rolle, Jos{\'e} Luis
and Mart{\'i}nez de Pis{\'o}n, Francisco Javier
and Garc{\'i}a Bringas, Pablo
and Mart{\'i}nez {\'A}lvarez, Francisco
and Herrero, {\'A}lvaro
and Fosci, Paolo
and S{\'e}rgio Filipe, Ramos",
title="Perplexity, Uncertainty, and the Limits of Active Learning",
booktitle="Hybrid Artificial Intelligent Systems",
year="2026",
publisher="Springer Nature Switzerland",
address="Cham",
pages="315--327",
isbn="978-3-032-08462-0"
}

@misc{ren2023outofdistributiondetectionselectivegeneration,
      title={Out-of-Distribution Detection and Selective Generation for Conditional Language Models}, 
      author={Jie Ren and Jiaming Luo and Yao Zhao and Kundan Krishna and Mohammad Saleh and Balaji Lakshminarayanan and Peter J. Liu},
      year={2023},
      eprint={2209.15558},
      archivePrefix={arXiv},
      primaryClass={cs.CL},
      url={https://arxiv.org/abs/2209.15558}, 
}

@inproceedings{10.1145/3706598.3714069,
author = {Alvarado Garcia, Adriana and Candello, Heloisa and Badillo-Urquiola, Karla and Wong-Villacres, Marisol},
title = {Emerging Data Practices: Data Work in the Era of Large Language Models},
year = {2025},
isbn = {9798400713941},
publisher = {Association for Computing Machinery},
address = {New York, NY, USA},
url = {https://doi.org/10.1145/3706598.3714069},
doi = {10.1145/3706598.3714069},
booktitle = {Proceedings of the 2025 CHI Conference on Human Factors in Computing Systems},
articleno = {846},
numpages = {21},
location = {
},
series = {CHI '25}
}

@misc{devlin2019bertpretrainingdeepbidirectional,
      title={BERT: Pre-training of Deep Bidirectional Transformers for Language Understanding}, 
      author={Jacob Devlin and Ming-Wei Chang and Kenton Lee and Kristina Toutanova},
      year={2019},
      eprint={1810.04805},
      archivePrefix={arXiv},
      primaryClass={cs.CL},
      url={https://arxiv.org/abs/1810.04805}, 
}

@misc{ahmed2024studyingllmperformanceclosed,
      title={Studying LLM Performance on Closed- and Open-source Data}, 
      author={Toufique Ahmed and Christian Bird and Premkumar Devanbu and Saikat Chakraborty},
      year={2024},
      eprint={2402.15100},
      archivePrefix={arXiv},
      primaryClass={cs.SE},
      url={https://arxiv.org/abs/2402.15100}, 
}

@article{phillips2007tamed,
  author       = {Phillips, Ruth},
  title        = {Tamed or Trained? The Co-option and Capture of `Favoured' NGOs},
  journal      = {Third Sector Review},
  volume       = {13},
  number       = {2},
  year         = {2007},
  month        = {July},
  pages        = {27+},

  note         = {Accessed 28 Sept.\ 2025},
}

@article{bf5a5d3b-3d24-3822-bb09-28237614c31e,
 ISSN = {00129682, 14680262},
 URL = {http://www.jstor.org/stable/1906951},
 author = {John Nash},
 journal = {Econometrica},
 number = {1},
 pages = {128--140},
 publisher = {[Wiley, Econometric Society]},
 title = {Two-Person Cooperative Games},
 urldate = {2025-10-03},
 volume = {21},
 year = {1953}
}

@online{hf_llama_3_8b_instruct,
  author       = {{Hugging Face}},
  title        = {Meta-Llama-3-8B-Instruct},
  year         = {2024},
  url          = {https://huggingface.co/meta-llama/Meta-Llama-3-8B-Instruct},
  note         = {Accessed: 2025-10-08}
}

@online{meta_llama_3_1,
  author       = {{Meta AI}},
  title        = {Introducing Meta Llama 3.1: Our most capable models to date},
  year         = {2024},
  url          = {https://ai.meta.com/blog/meta-llama-3-1/},
  note         = {Accessed: 2025-10-08}
}

@misc{biderman2023pythiasuiteanalyzinglarge,
      title={Pythia: A Suite for Analyzing Large Language Models Across Training and Scaling}, 
      author={Stella Biderman and Hailey Schoelkopf and Quentin Anthony and Herbie Bradley and Kyle O'Brien and Eric Hallahan and Mohammad Aflah Khan and Shivanshu Purohit and USVSN Sai Prashanth and Edward Raff and Aviya Skowron and Lintang Sutawika and Oskar van der Wal},
      year={2023},
      eprint={2304.01373},
      archivePrefix={arXiv},
      primaryClass={cs.CL},
      url={https://arxiv.org/abs/2304.01373}, 
}

@article{doi:10.1177/0007650312472609,
author = {Heidi Herlin},
title ={Better Safe Than Sorry: Nonprofit Organizational Legitimacy and Cross-Sector Partnerships},

journal = {Business \& Society},
volume = {54},
number = {6},
pages = {822-858},
year = {2015},
doi = {10.1177/0007650312472609},

URL = { 
    
        https://doi.org/10.1177/0007650312472609
    
    

},
eprint = { 
    
        https://doi.org/10.1177/0007650312472609
    
    

}
}

@article{AI2022103101,
title = {Effects of offering incentives for reviews on trust: Role of review quality and incentive source},
journal = {International Journal of Hospitality Management},
volume = {100},
pages = {103101},
year = {2022},
issn = {0278-4319},
doi = {https://doi.org/10.1016/j.ijhm.2021.103101},
url = {https://www.sciencedirect.com/science/article/pii/S0278431921002449},
author = {Jin Ai and Dogan Gursoy and Yue Liu and Xingyang Lv}
}

@misc{golazizian2024costefficientsubjectivetaskannotation,
      title={Cost-Efficient Subjective Task Annotation and Modeling through Few-Shot Annotator Adaptation}, 
      author={Preni Golazizian and Alireza S. Ziabari and Ali Omrani and Morteza Dehghani},
      year={2024},
      eprint={2402.14101},
      archivePrefix={arXiv},
      primaryClass={cs.CL},
      url={https://arxiv.org/abs/2402.14101}, 
}

@INPROCEEDINGS{6889457,
  author={Wang, Dan and Shang, Yi},
  booktitle={2014 International Joint Conference on Neural Networks (IJCNN)}, 
  title={A new active labeling method for deep learning}, 
  year={2014},
  volume={},
  number={},
  pages={112-119},
  doi={10.1109/IJCNN.2014.6889457},
  ISSN={2161-4407},
  month={July},}

@INPROCEEDINGS{10068925,
  author={Gharawi, Abdulrahman Ahmed and Alsubhi, Jumana and Ramaswamy, Lakshmish},
  booktitle={2022 21st IEEE International Conference on Machine Learning and Applications (ICMLA)}, 
  title={Impact of Labeling Noise on Machine Learning: A Cost-aware Empirical Study}, 
  year={2022},
  volume={},
  number={},
  pages={936-939},
  doi={10.1109/ICMLA55696.2022.00156}}

@article{meta-analytic,
author = {Deci, Edward and Koestner, Richard and Ryan, Richard},
year = {1999},
month = {11},
pages = {627-668},
title = {A Meta-Analytic Review of Experiments Examining the Effects of Extrinsic Rewards on Intrinsic Motivation},
volume = {125},
journal = {Psychological Bulletin},
doi = {10.1037/0033-2909.125.6.627}
}

@article{FELGENHAUER2024112017,
title = {Competition for publication-based rewards},
journal = {Economics Letters},
volume = {244},
pages = {112017},
year = {2024},
issn = {0165-1765},
doi = {https://doi.org/10.1016/j.econlet.2024.112017},
url = {https://www.sciencedirect.com/science/article/pii/S0165176524005019},
author = {Mike Felgenhauer}
}

@article{pile,
  title={The {P}ile: An 800GB Dataset of Diverse Text for Language Modeling},
  author={Gao, Leo and Biderman, Stella and Black, Sid and Golding, Laurence and Hoppe, Travis and Foster, Charles and Phang, Jason and He, Horace and Thite, Anish and Nabeshima, Noa and Presser, Shawn and Leahy, Connor},
  journal={arXiv preprint arXiv:2101.00027},
  year={2020}
}

@misc{almazrouei2023falconseriesopenlanguage,
      title={The Falcon Series of Open Language Models}, 
      author={Ebtesam Almazrouei and Hamza Alobeidli and Abdulaziz Alshamsi and Alessandro Cappelli and Ruxandra Cojocaru and Mérouane Debbah and Étienne Goffinet and Daniel Hesslow and Julien Launay and Quentin Malartic and Daniele Mazzotta and Badreddine Noune and Baptiste Pannier and Guilherme Penedo},
      year={2023},
      eprint={2311.16867},
      archivePrefix={arXiv},
      primaryClass={cs.CL},
      url={https://arxiv.org/abs/2311.16867}, 
}

@misc{dai2024deepseekmoeultimateexpertspecialization,
      title={DeepSeekMoE: Towards Ultimate Expert Specialization in Mixture-of-Experts Language Models}, 
      author={Damai Dai and Chengqi Deng and Chenggang Zhao and R. X. Xu and Huazuo Gao and Deli Chen and Jiashi Li and Wangding Zeng and Xingkai Yu and Y. Wu and Zhenda Xie and Y. K. Li and Panpan Huang and Fuli Luo and Chong Ruan and Zhifang Sui and Wenfeng Liang},
      year={2024},
      eprint={2401.06066},
      archivePrefix={arXiv},
      primaryClass={cs.CL},
      url={https://arxiv.org/abs/2401.06066}, 
}

@article{article,
author = {Reischauer, Georg and Mair, Johanna},
year = {2018},
month = {09},
pages = {220–247},
title = {How Organizations Strategically Govern Online Communities: Lessons from the Sharing Economy},
volume = {4},
journal = {Academy of Management Discoveries},
doi = {10.5465/amd.2016.0164}
}

@misc{hou2024largelanguagemodelssoftware,
      title={Large Language Models for Software Engineering: A Systematic Literature Review}, 
      author={Xinyi Hou and Yanjie Zhao and Yue Liu and Zhou Yang and Kailong Wang and Li Li and Xiapu Luo and David Lo and John Grundy and Haoyu Wang},
      year={2024},
      eprint={2308.10620},
      archivePrefix={arXiv},
      primaryClass={cs.SE},
      url={https://arxiv.org/abs/2308.10620}, 
}

@misc{lu2024deepseekvlrealworldvisionlanguageunderstanding,
      title={DeepSeek-VL: Towards Real-World Vision-Language Understanding}, 
      author={Haoyu Lu and Wen Liu and Bo Zhang and Bingxuan Wang and Kai Dong and Bo Liu and Jingxiang Sun and Tongzheng Ren and Zhuoshu Li and Hao Yang and Yaofeng Sun and Chengqi Deng and Hanwei Xu and Zhenda Xie and Chong Ruan},
      year={2024},
      eprint={2403.05525},
      archivePrefix={arXiv},
      primaryClass={cs.AI},
      url={https://arxiv.org/abs/2403.05525}, 
}

@article{scientificreports2024genai,
  author       = {Zhang, Yucheng and Sun, Chenhao and Chen, Jiajun and Wang, Yuxin and Zhang, Yongfeng},
  title        = {The consequences of generative AI for online knowledge communities},
  journal      = {Scientific Reports},
  year         = {2024},
  volume       = {14},
  number       = {16321},
  doi          = {10.1038/s41598-024-69804-5},
  publisher    = {Nature Publishing Group},
}

@article{10.1108/eb026526,
    author = {SPARCK JONES, KAREN},
    title = {A STATISTICAL INTERPRETATION OF TERM SPECIFICITY AND ITS APPLICATION IN RETRIEVAL},
    journal = {Journal of Documentation},
    volume = {28},
    number = {1},
    pages = {11-21},
    year = {1972},
    month = {01},
    issn = {0022-0418},
    doi = {10.1108/eb026526},
    url = {https://doi.org/10.1108/eb026526},
    eprint = {https://www.emerald.com/jd/article-pdf/28/1/11/1336479/eb026526.pdf},
}

@misc{beau2024codeinsightcurateddatasetpractical,
      title={CodeInsight: A Curated Dataset of Practical Coding Solutions from Stack Overflow}, 
      author={Nathanaël Beau and Benoît Crabbé},
      year={2024},
      eprint={2409.16819},
      archivePrefix={arXiv},
      primaryClass={cs.CL},
      url={https://arxiv.org/abs/2409.16819}, 
}

@article{impacts,
author = {Li, Xinyu and Kim, Keongtae},
year = {2024},
month = {09},
pages = {577-591},
title = {Impacts of generative AI on user contributions: evidence from a coding Q\&A platform},
volume = {36},
journal = {Marketing Letters},
doi = {10.1007/s11002-024-09747-1}
}

@article{procaccia2013approximate,
  title={Approximate mechanism design without money},
  author={Procaccia, Ariel D and Tennenholtz, Moshe},
  journal={ACM Transactions on Economics and Computation (TEAC)},
  volume={1},
  number={4},
  pages={1--26},
  year={2013},
  publisher={ACM New York, NY, USA}
}

@article{koiliaris2018subset,
  title={Subset sum made simple},
  author={Koiliaris, Konstantinos and Xu, Chao},
  journal={arXiv preprint arXiv:1807.08248},
  year={2018}
}

@inproceedings{FainNashIsNP,
author = {Fain, Brandon and Munagala, Kamesh and Shah, Nisarg},
title = {Fair Allocation of Indivisible Public Goods},
year = {2018},
isbn = {9781450358293},
publisher = {Association for Computing Machinery},
address = {New York, NY, USA},
url = {https://doi.org/10.1145/3219166.3219174},
doi = {10.1145/3219166.3219174},
booktitle = {Proceedings of the 2018 ACM Conference on Economics and Computation},
pages = {575–592},
numpages = {18},
location = {Ithaca, NY, USA},
series = {EC '18}
}

@article{helic2025stack,
  title={Stack Overflow Is Not Dead Yet: Crowd Answers Still Matter},
  author={Helic, Denis and Santos, Tiago},
  journal={arXiv preprint arXiv:2509.05879},
  year={2025}
}

@article{forsythe1994fairness,
  title={Fairness in simple bargaining experiments},
  author={Forsythe, Robert and Horowitz, Joel L and Savin, Nathan E and Sefton, Martin},
  journal={Games and Economic behavior},
  volume={6},
  number={3},
  pages={347--369},
  year={1994},
  publisher={Elsevier}
}

@inproceedings{wu2025munba,
  title={Munba: Machine unlearning via nash bargaining},
  author={Wu, Jing and Harandi, Mehrtash},
  booktitle={Proceedings of the IEEE/CVF International Conference on Computer Vision},
  pages={4754--4765},
  year={2025}
}

@article{zeng2024fairness,
  title={Fairness-aware meta-learning via nash bargaining},
  author={Zeng, Yi and Yang, Xuelin and Chen, Li and Ferrer, Cristian and Jin, Ming and Jordan, Michael and Jia, Ruoxi},
  journal={Advances in Neural Information Processing Systems},
  volume={37},
  pages={83235--83267},
  year={2024}
}

@inproceedings{zhang2025fairshare,
  title={Fairshare Data Pricing via Data Valuation for Large Language Models},
  author={Zhang, Luyang and Jiao, Cathy and Li, Beibei and Xiong, Chenyan},
  booktitle={Advances in Neural Information Processing Systems},
  year={2025}
}

@inproceedings{carlini2021extracting,
  title={Extracting training data from large language models},
  author={Carlini, Nicholas and Tramer, Florian and Wallace, Eric and Jagielski, Matthew and Herbert-Voss, Ariel and Lee, Katherine and Roberts, Adam and Brown, Tom and Song, Dawn and Erlingsson, Ulfar and others},
  booktitle={30th USENIX security symposium (USENIX Security 21)},
  pages={2633--2650},
  year={2021}
}

@article{mustafa2023motivates,
  title={What motivates online community contributors to contribute consistently? A case study on Stackoverflow netizens},
  author={Mustafa, Sohaib and Zhang, Wen and Naveed, Muhammad Mateen},
  journal={Current Psychology},
  volume={42},
  number={13},
  pages={10468--10481},
  year={2023},
  publisher={Springer}
}

@article{lu2018learning,
  title={Learning under concept drift: A review},
  author={Lu, Jie and Liu, Anjin and Dong, Fan and Gu, Feng and Gama, Joao and Zhang, Guangquan},
  journal={IEEE transactions on knowledge and data engineering},
  volume={31},
  number={12},
  pages={2346--2363},
  year={2018},
  publisher={IEEE}
}

\ifnum\Includeappendix=1
\onecolumn
\appendix

\section{Appendix}

\section{NP-Hardness of Maximizing Equation~\eqref{eq:Nash product}}

\paragraph{Problem Statement.}
Let $\Omega = \{1,\dots,n\}$ be a finite set, and let $f,g:\Omega\to\mathbb{R}_{\geq 0}$ be two non-negative functions.  Given an integer $k$, consider the optimization problem
\begin{equation}\label{eq:nash product problem}
\max_{S\subseteq\Omega,\;|S|\le k} \;\Bigl(\sum_{i\in S}f(i)\Bigr)\;\Bigl(\sum_{i\in S}g(i)\Bigr).
\end{equation}

\begin{theorem}\label{thm:hard main}
The problem in Equation~\eqref{eq:nash product problem} is NP-hard.
\end{theorem}
\begin{proof}
We reduce from the \emph{cardinality-constrained subset sum} problem (CCSS):
\begin{itemize}
  \item \emph{Instance}: positive integers $a_1,\dots,a_n$, a target $T$, and size bound $k$.
  \item \emph{Question}: is there a subset $S\subseteq\{1,\dots,n\}$ of size $|S|=k$ such that $\sum_{i\in S}a_i = T$?
\end{itemize}
\begin{proposition}\label{prop:CCSS}
    The CCSS problem is NP-hard.
\end{proposition}
\begin{proof}[Proof of Proposition~\ref{prop:CCSS}]
We prove this proposition by showing that CCSS is at least as hard as the Subset Sum problem, which is known to be NP-hard~\cite{koiliaris2018subset}. Given an instance \((a_1,\dots,a_n;T)\) of Subset Sum, form the multiset \(\{a_1,\dots,a_n,0,\dots,0\}\) with \(n\) zeros and set \(K=n\).  Then any size-\(K\) subset summing to \(T\) picks exactly the original items that sum to \(T\) (and fills the rest with zeros), so Subset Sum reduces in polynomial time to Exact-\(K\) Subset Sum.  
\end{proof}
We continue with the proof of Theorem~\ref{thm:hard main}. On this input, define an instance of our bilinear maximization by choosing
\[
f(i)=a_i,\quad
M=\frac{2T}{k},\quad
g(i)=M - a_i
\quad\text{for each }i=1,\dots,n.
\]

For any subset $S$ with $|S|=k$, let
\[
A_S = \sum_{i\in S}a_i.
\]
Then
\[
\sum_{i\in S}f(i) = A_S,
\qquad
\sum_{i\in S}g(i) = kM - A_S,
\]
and thus
\[
F(S)
= A_S\,(kM - A_S)
= -A_S^2 + kM\,A_S.
\]
The quadratic function
\[
Q(x) = -x^2 + kM\,x
\]
is strictly concave in $x$ and attains its unique maximum at
\[
x^* = \frac{kM}{2} = \frac{k}{2}\cdot\frac{2T}{k} = T.
\]
Hence, among all subsets of size $k$, $F(S)$ is \emph{uniquely maximized} exactly when $A_S=T$.  Therefore, solving the bilinear maximization and checking whether the optimum equals
\[
T\,(kM - T) = T\,(2T - T)=T^2
\]
answers the subset-sum question.  Since the latter is NP-hard, our maximization problem is NP-hard as well. 
\end{proof}

\section{Limitations}\label{sec:limitations}
Our work makes several assumptions about the interaction between online forums and LLM providers. These assumptions are chosen to isolate the core strategic tradeoffs in a minimal, analytically tractable setting. We acknowledge that they may not hold in real-world environments. In particular, we highlight the following assumptions and corresponding limitations of our framework:

\begin{itemize}[leftmargin=*]
    \item \textbf{Two-player game:} Our work is conceptual with the goal of presenting a proof-of-concept collaboration framework between a single GenAI company and a single forum. However, practical scenarios could involve multiple GenAI entities and multiple forums simultaneously. Such multi-agent settings could introduce strategic interdependence among GenAI players; for example, if an LLM company anticipates that a competitor will submit certain questions, it may adjust its own submissions strategically. Alternatively, if one forum refuses to publish questions from a certain LLM provider, the provider could decide to discontinue collaboration with that forum.

    \item \textbf{Single domain of utility proxies:} Our analysis is limited to a single domain of utility proxies. We do not examine cases where the GenAI company submits domain-specific content (e.g., questions related to its own products or models) or where the forum's selection rule $\mathcal{R}$ evolves dynamically over time. Such extensions could meaningfully affect both the utility structure and the observed patterns of cooperation.

    \item \textbf{Temporal scope:} The empirical evaluation is based on data covering a one-year period. This temporal scope may not fully capture the long-term stability or potential cyclical behavior of cooperative strategies.

    \item \textbf{Linear utilities:} Recall that our model assumes utilities are linear. There are two ways to justify this modeling decision. First, linear aggregation enables tractable analysis and transparent comparison between strategies, avoiding the computational complexity of modeling higher-order dependencies. Second, since the number of questions published through Player~F represents only a small fraction of the overall data available to Player~G, potential interactions between questions, such as redundancy or topical overlap, can be reasonably neglected. 
    
    While this simplifies the setting, it might be unrealistic in some scenarios. Indeed, publishing two similar questions in a forum results in cannibalization of view counts, as the total engagement may be higher if only one of them were published. Non-linear utility functions (e.g., the submodular formulation analyzed by \citet{10.5555/3709347.3743523}) may better capture redundancy effects and are left for future work.
\end{itemize}

\section{Incentive Misalignment}\label{sec:appendix:Incentive mis}
Recall the incentive misalignment analysis from Subsection~\ref{subsec:findings mis}. In this section, we examine the Spearman correlation between the metrics representing the players' incentives: weekly normalized view counts versus perplexity. The data reveals a consistent trend across all models and domains: the correlation is non-existent to slightly negative, with statistical significance across all evaluated settings.

\begin{table}[t]
\small
\centering
\caption{Spearman correlation between normalized view count and perplexity for each domain. Asterisks denote statistical significance at the $0.01$ level.}
\label{spearman_correlation_and_p_val}
\begin{tabular}{l l r r}
\hline
\textbf{LLM} & \textbf{Domain} & \textbf{Spearman} & \textbf{p-value} \\
\hline
Pythia 6.9B & ubuntu         & -0.11725 & $8.6078\times10^{-37}{}^*$  \\
Pythia 6.9B & english        & -0.05326 & $3.81\times10^{-2}$         \\
Pythia 6.9B & latex          & -0.04045 & $2\times10^{-4}{}^*$        \\
Pythia 6.9B & stackoverflow  & -0.01343 & $1.9556\times10^{-11}{}^*$  \\
Pythia 6.9B & math           & -0.07459 & $2.2311\times10^{-54}{}^*$  \\
\hline
LLaMA 3.1 8B & ubuntu        & -0.12704 & $5.7306\times10^{-43}{}^*$  \\
LLaMA 3.1 8B & english       & -0.08214 & $1.2\times10^{-3}{}^*$      \\
LLaMA 3.1 8B & latex         & -0.02481 & $2.27\times10^{-2}$         \\
LLaMA 3.1 8B & stackoverflow & -0.04896 & $2.2188\times10^{-132}{}^*$ \\
LLaMA 3.1 8B & math          & -0.08519 & $1.5570\times10^{-70}{}^*$  \\
\hline
LLaMA 3.1 8B-Instruct & ubuntu        & -0.11300 & $2.6493\times10^{-34}{}^*$  \\
LLaMA 3.1 8B-Instruct & english       & -0.05642 & $2.67\times10^{-2}$         \\
LLaMA 3.1 8B-Instruct & latex         & -0.03429 & $1.6\times10^{-3}{}^*$      \\
LLaMA 3.1 8B-Instruct & stackoverflow & -0.04468 & $1.5948\times10^{-110}{}^*$ \\
LLaMA 3.1 8B-Instruct & math          & -0.08446 & $2.3419\times10^{-69}{}^*$  \\
\hline
\end{tabular}
\end{table}
\section{Classifier Performance and Temporal Degradation}
\label{app:classifier_performance}

Recall from Section~\ref{subsec:strategies} that Player~F uses an engagement classifier to predict which questions are likely to attract high user engagement. In this section, we examine how the classifier's performance evolves over time on the weekly evaluation datasets used throughout the experiment. Since the distribution of Stack Exchange questions shifts over time, the classifier is naturally affected by temporal distribution shift (e.g. concept drift~\cite{lu2018learning}). As a result, a model trained once and left unchanged may become increasingly misaligned with the evolving linguistic patterns, topical trends, and user behaviors that drive engagement.

While the classifier achieves strong performance initially, we observe substantial temporal degradation when the model is evaluated over an extended period without retraining. Figure~\ref{fig:roc_decline} shows the ROC-AUC of a static BERT classifier trained only once at the beginning of the observation period. Over the course of 52 weeks, its performance declines steadily and sharply, dropping from an initial ROC-AUC of approximately $0.76$ to below $0.45$. This deterioration indicates that the signals governing forum engagement evolve rapidly enough that a static classifier becomes increasingly ineffective over time.

To address this degradation, we adopt a periodic retraining strategy. As shown in Figure~\ref{fig:roc_decline}, we retrain the classifier every quarter, namely every 13 weeks, and mark these retraining points with stars. Under this policy, the classifier maintains substantially more stable predictive performance throughout the experiment. In particular, the ROC-AUC generally remains in the range of $0.65$ to $0.70$, with clear recoveries at the retraining points (weeks 13, 26, and 39, marked with a red star), where performance rises to around $0.80$. Overall, this comparison demonstrates that periodic retraining is essential for maintaining a robust and reliable engagement predictor in a dynamic forum environment.

\begin{figure}[t]
    \centering
    % Replace with your actual image path
    \includegraphics[width=0.98\textwidth]{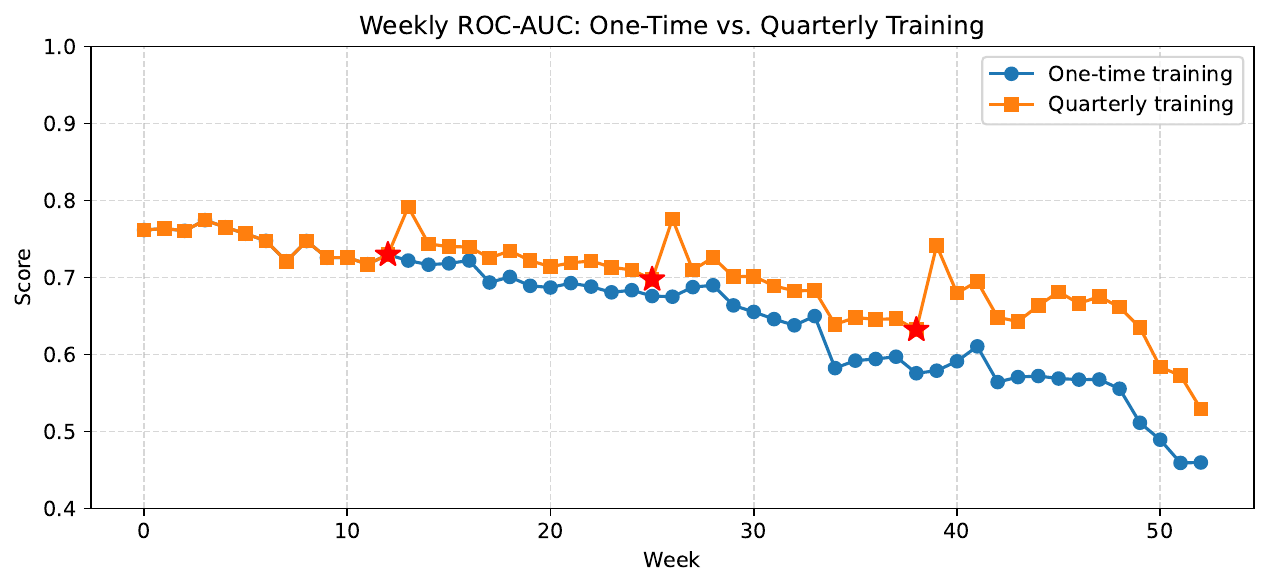} 
    \caption{Weekly ROC-AUC over the 52-week evaluation horizon. The horizontal axis shows the week index, and the vertical axis shows the classifier's ROC-AUC on that week's evaluation set. The blue curve (circle markers) corresponds to a classifier trained only once at the beginning of the period, while the orange curve (square markers) corresponds to a classifier retrained quarterly on cumulative data. Red star markers indicate the retraining weeks. The one-time-trained classifier exhibits a gradual decline in predictive performance over time, reflecting temporal distribution shift. In contrast, quarterly retraining yields consistently stronger performance, with clear upward jumps in ROC-AUC immediately after each retraining point, indicating that refreshing the model with recent data substantially restores predictive accuracy.}
    \label{fig:roc_decline}
\end{figure}

\section{Representative Stack Overflow Examples}
\label{sec:example questions}
\newcounter{qcounter}
\renewcommand{\theqcounter}{Q\arabic{qcounter}}
\newcommand{\topfiveper}{\ensuremath{(H)}}
\newcommand{\bottomfiveper}{\ensuremath{(L)}}

\begin{table}[t!]
\centering
\caption{Representative questions from the four extreme combinations of perplexity and normalized views, where \topfiveper\ and \bottomfiveper\ indicate that the corresponding metric falls in the top 5\% and bottom 5\%, respectively. The groups are high-high (\ref{q:79632861}--\ref{q:79627485}), low-low (\ref{q:79041284}--\ref{q:79029955}), low-high (\ref{q:79631649}--\ref{q:79629891}), and high-low (\ref{q:79041253}--\ref{q:79030629}).}
\label{tab:extremes_questions}
\footnotesize
\setlength{\tabcolsep}{5pt}
\begin{tabular}{l p{8cm} c c c}
\toprule
\textbf{Index} & \textbf{Question Title} & \textbf{Date} & \textbf{Perplexity} & \textbf{Views} \\
\midrule
\refstepcounter{qcounter}\theqcounter \label{q:79632861} & ``API Error: 400 invalid beta flag'' when trying to use Claude Code with Bedrock using claude 3.5 haiku & 2025-05-21 & 70.2~\topfiveper & $2.0 \cdot 10^{-3}$~\topfiveper \\
\refstepcounter{qcounter}\theqcounter \label{q:79631531} & LangGraph resume after interrupt is not working properly when running with more than 1 worker in uvicorn & 2025-05-10 & 62.8~\topfiveper & $1.3 \cdot 10^{-3}$~\topfiveper \\
\refstepcounter{qcounter}\theqcounter \label{q:79630450} & OpenAI Functions Tool calling not actually calling for function & 2025-05-10 & 56.17~\topfiveper & $10^{-4}$~\topfiveper \\
\refstepcounter{qcounter}\theqcounter \label{q:79628471} & MCP Server is giving ``The selected tool does not have a callable `function''' when using via autogen & 2025-05-04 & 61.8~\topfiveper & $6.3 \cdot 10^{-3}$~\topfiveper \\
\refstepcounter{qcounter}\theqcounter \label{q:79627485} & ollama.generate raises model not found error: ``hf.co/mradermacher/Llama-3.2-3B-Instruct-uncensored-GGUF'' & 2025-05-04 & 140.41~\topfiveper & $1.1 \cdot 10^{-3}$~\topfiveper \\
\refstepcounter{qcounter}\theqcounter \label{q:79041284} & Python: only length-1 arrays can be converted to python scalars error when trying to plot & 2024-09-30 & 7.0~\bottomfiveper & $1.4 \cdot 10^{-5}$~\bottomfiveper \\
\refstepcounter{qcounter}\theqcounter \label{q:79037036} & JPQL nested join & 2024-09-29 & 3.6~\bottomfiveper & $2.1 \cdot 10^{-5}$~\bottomfiveper \\
\refstepcounter{qcounter}\theqcounter \label{q:79035799} & How to Sort Alphanumeric Values? & 2024-09-29 & 5.2~\bottomfiveper & $1.8 \cdot 10^{-5}$~\bottomfiveper \\
\refstepcounter{qcounter}\theqcounter \label{q:79031646} & Why is nginx ignores .js files, but all other endpoints work good? & 2024-09-27 & 7.4~\bottomfiveper & $1.9 \cdot 10^{-5}$~\bottomfiveper \\
\refstepcounter{qcounter}\theqcounter \label{q:79029955} & Issue when using routerLinkActive in parent route with dynamically generated child route & 2024-09-27 & 6.8~\bottomfiveper & $1.5 \cdot 10^{-5}$~\bottomfiveper \\
\refstepcounter{qcounter}\theqcounter \label{q:79631649} & How can I find the largest axis-aligned rectangle of a given aspect ratio that can fit inside a rotated rectangle? & 2025-05-21 & 4.9~\bottomfiveper & $8.3 \cdot 10^{-4}$~\topfiveper \\
\refstepcounter{qcounter}\theqcounter \label{q:79631476} & Can you explicitly provide template arguments to std::gcd? & 2025-05-21 & 6.9~\bottomfiveper & $6.7 \cdot 10^{-4}$~\topfiveper \\
\refstepcounter{qcounter}\theqcounter \label{q:79631106} & What's the difference between \#\texttt{[repr(Rust)]}, \#\texttt{[repr(C)]} and \#\texttt{[repr(packed)]}? & 2025-05-20 & 4.3~\bottomfiveper & $1.6 \cdot 10^{-3}$~\topfiveper \\
\refstepcounter{qcounter}\theqcounter \label{q:79630050} & Powershell script closes immediately when run, but runs fine when done manually & 2025-05-20 & 7.2~\bottomfiveper & $7.9 \cdot 10^{-4}$~\topfiveper \\
\refstepcounter{qcounter}\theqcounter \label{q:79629891} & Why Tailwind is not recognized as an internal or external command, operable program or batch file? & 2025-05-20 & 5.9~\bottomfiveper & $1.1 \cdot 10^{-3}$~\topfiveper \\
\refstepcounter{qcounter}\theqcounter \label{q:79041253} & Python code to trim HTML content without losing html structure & 2024-09-30 & 53.5~\topfiveper & $2.0 \cdot 10^{-5}$~\bottomfiveper \\
\refstepcounter{qcounter}\theqcounter \label{q:79031842} & ``@ModelAttribute returns null when passing JSON in Spring Boot POST request'' & 2024-09-27 & 52.1~\topfiveper & $1.9 \cdot 10^{-5}$~\bottomfiveper \\
\refstepcounter{qcounter}\theqcounter \label{q:79031749} & Elasticsearch Match Specified & 2024-09-27 & 63.1~\topfiveper & $1.4 \cdot 10^{-5}$~\bottomfiveper \\
\refstepcounter{qcounter}\theqcounter \label{q:79030647} & GitHub uploaded file output trouble & 2024-09-27 & 74.0~\topfiveper & $1.9 \cdot 10^{-5}$~\bottomfiveper \\
\refstepcounter{qcounter}\theqcounter \label{q:79030629} & 3DS OutScale programmatic user & 2024-09-27 & 61.7~\topfiveper & $1.1 \cdot 10^{-5}$~\bottomfiveper \\
\bottomrule
\end{tabular}
\end{table}

As described in Section~\ref{subsec:data int}, our experimental corpus consists of Stack Exchange questions collected across multiple domains over the period July 2024 to July 2025. In this section, we provide illustrative examples from that corpus, focusing on representative \texttt{stackoverflow} questions that help ground the utility functions. To better understand the dynamics of forum engagement and the requirements for high-quality GenAI training data, we examine ``extreme'' questions evaluated by Pythia 6.9B, sampled from the top and bottom 5\% quantiles of both perplexity and normalized view counts. As shown in Table~\ref{tab:extremes_questions}, there is a stark topical contrast between these extremes. Taken together, these examples highlight the tension between model predictability and forum engagement through distinct questions drawn from different combinations of the two quantities, and provide a qualitative complement to our aggregate analysis by showing the concrete types of questions that populate each region of the distribution.

The high perplexity and high views group (\ref{q:79632861}--\ref{q:79627485}) is entirely dominated by novel, cutting-edge Generative AI frameworks and tooling. Questions detailing integrations with Claude 3.5 (\ref{q:79632861}), LangGraph (\ref{q:79631531}), OpenAI function calling (\ref{q:79630450}), MCP servers (\ref{q:79628471}), and Ollama (\ref{q:79627485}) demonstrate highly modern, specialized terminology. These emerging technologies generate massive community interest (driving high views) but contain vocabulary and paradigms that the LLM struggles to predict (resulting in high perplexity).

Conversely, the low perplexity and low views group (\ref{q:79041284}--\ref{q:79029955}) revolves around standard, deeply established programming and infrastructure concepts. Queries regarding fundamental Python array errors (\ref{q:79041284}), JPQL joins (\ref{q:79037036}), alphanumeric sorting (\ref{q:79035799}), Nginx configuration (\ref{q:79031646}), and front-end routing (\ref{q:79029955}) are highly common in training corpora. Because these topics are foundational, they are highly recognizable and predictable to the LLM (yielding low perplexity), while simultaneously generating minimal engagement or new interest from the broader forum community.

The questions in the low perplexity and high views group (\ref{q:79631649}--\ref{q:79629891}) represent heavily documented challenges that span across fundamental computer science and daily developer friction. We see core language mechanics like Rust's memory representation (\ref{q:79631106}) and C++ template arguments (\ref{q:79631476}), algorithmic geometry such as fitting a rectangle inside a rotated rectangle (\ref{q:79631649}), and incredibly common environment setup errors like PowerShell execution policies (\ref{q:79630050}) and Tailwind command recognition (\ref{q:79629891}). Because these issues are extensively debated and resolved in tutorials and documentation, the LLM predicts their linguistic patterns with high confidence (yielding low perplexity). Concurrently, they draw massive traffic because they act as universal stumbling blocks for a vast number of active programmers.

The last quadrant of high perplexity and low views questions (\ref{q:79041253}--\ref{q:79030629}) captures the ``long tail'' of forum queries, characterized by extreme specificity, unpopular tooling, or irregular phrasing (\ref{q:79041253}). Queries involving niche platforms like "3DS OutScale programmatic user" (\ref{q:79030629}) or hyper-specific framework collisions, such as Spring Boot @ModelAttribute returning null specifically for JSON POST requests (\ref{q:79031842}), lack the broad appeal needed to generate views. Furthermore, titles like "Elasticsearch Match Specified" (\ref{q:79031749}) or "GitHub uploaded file output trouble" (\ref{q:79030647}) are syntactically fragmented and underspecified. This combination of rare terminology and unpredictable sentence structure severely disrupts the model's predictive capabilities, driving up perplexity for questions that the broader community largely ignores.

\section{Statistical Analysis}
In this section, we provide an exhaustive statistical analysis of our findings.
\subsection{Full Collaboration}
\label{sec:Full_colab}
We begin with the full collaboration scenario, considering all heuristics described in Subsection~\ref{subsec:heuristic}. Table~\ref{appendix:tab:full_coop_perp} reports the utility of Player~G, measured by perplexity. The results show that all heuristics consistently outperform the random baseline with a high level of statistical significance. Moreover, the \textbf{GreedyNP} heuristic achieves the best overall performance for Player~G.

Table~\ref{appendix:tab:llm-views-comparison} analyzes Player~F's utility, measured by normalized view counts. Here as well, all baselines significantly outperform the random strategy, while \textbf{MPP} achieves the best results for Player~F.

\begin{table*}[t]
\small
\caption{
Significance analysis of Player~G's utility (perplexity) under full collaboration. The table reports pairwise comparisons between heuristics using a weekly mean--based t-test over 52 weeks, for each evaluated LLM. T-statistics values are shown for each pair of methods, together with the corresponding p-values. Asterisks denote statistical significance at the $0.01$ level. All heuristics significantly outperform the random baseline, while differences show \textbf{GreedyNP} outperforms systematically.
}
\label{appendix:tab:full_coop_perp}
\footnotesize
\centering
\setlength{\tabcolsep}{8pt}
\renewcommand{\arraystretch}{1.1}
\begin{tabular}{l l l r l}
\toprule
\textbf{LLM} & \textbf{Alg A} & \textbf{Alg B} & \textbf{t-statistic} & \textbf{p-value} \\ \hline
Pythia 6.9B & MaxSP & Random & $247.162$ & $1.585 \times 10^{-77}*$ \\
Pythia 6.9B & MaxSP & GreedyNP & $-28.586$ & $3.204 \times 10^{-32}*$ \\
Pythia 6.9B & MaxSP & MPP & $72.841$ & $1.294 \times 10^{-51}*$ \\
Pythia 6.9B & Random & GreedyNP & $-264.489$ & $5.743 \times 10^{-79}*$ \\
Pythia 6.9B & Random & MPP & $-219.769$ & $4.983 \times 10^{-75}*$ \\
Pythia 6.9B & GreedyNP & MPP & $102.512$ & $7.727 \times 10^{-59}*$ \\
\hline
Llama-3.1-8B & MaxSP & Random & $274.533$ & $9.258 \times 10^{-80}*$ \\
Llama-3.1-8B & MaxSP & GreedyNP & $-58.577$ & $4.989 \times 10^{-47}*$ \\
Llama-3.1-8B & MaxSP & MPP & $60.343$ & $1.186 \times 10^{-47}*$ \\
Llama-3.1-8B & Random & GreedyNP & $-292.503$ & $4.150 \times 10^{-81}*$ \\
Llama-3.1-8B & Random & MPP & $-218.164$ & $7.133 \times 10^{-75}*$ \\
Llama-3.1-8B & GreedyNP & MPP & $98.639$ & $5.053 \times 10^{-58}*$ \\
\hline
Llama-3.1-8B-Instruct & MaxSP & Random & $267.713$ & $3.173 \times 10^{-79}*$ \\
Llama-3.1-8B-Instruct & MaxSP & GreedyNP & $-50.226$ & $8.284 \times 10^{-44}*$ \\
Llama-3.1-8B-Instruct & MaxSP & MPP & $58.581$ & $4.970 \times 10^{-47}*$ \\
Llama-3.1-8B-Instruct & Random & GreedyNP & $-309.043$ & $2.806 \times 10^{-82}*$ \\
Llama-3.1-8B-Instruct & Random & MPP & $-275.482$ & $7.818 \times 10^{-80}*$ \\
Llama-3.1-8B-Instruct & GreedyNP & MPP & $99.829$ & $2.817 \times 10^{-58}*$ \\
\hline
\bottomrule
\end{tabular}
\end{table*}

\begin{table}[H]
\footnotesize
\centering
\caption{
Significance analysis of Player~F's utility (normalized view counts) under full collaboration. The table reports pairwise comparisons between heuristics using a weekly mean--based t-test over 52 weeks, for each evaluated LLM. t-statistics are shown for each method, together with the corresponding p-values. Asterisks denote statistical significance at the $0.01$ level. All heuristics significantly outperform the random baseline, while \textbf{MPP} outperforms the other heuristics.
}
\label{appendix:tab:llm-views-comparison}
\begin{tabular}{l l l r l}
\toprule
\textbf{LLM} & \textbf{Alg A} & \textbf{Alg B} & \textbf{t-statistic} & \textbf{p-value} \\ \hline
Pythia 6.9B & MaxSP & Random & $25.342$ & $8.117 \times 10^{-30}*$ \\
Pythia 6.9B & MaxSP & GreedyNP & $2.605$ & $1.212 \times 10^{-2}$ \\
Pythia 6.9B & MaxSP & MPP & $-28.412$ & $4.245 \times 10^{-32}*$ \\
Pythia 6.9B & Random & GreedyNP & $-25.715$ & $4.167 \times 10^{-30}*$ \\
Pythia 6.9B & Random & MPP & $-38.447$ & $2.924 \times 10^{-38}*$ \\
Pythia 6.9B & GreedyNP & MPP & $-29.258$ & $1.091 \times 10^{-32}*$ \\
\hline
Llama-3.1-8B & MaxSP & Random & $17.800$ & $4.944 \times 10^{-23}*$ \\
Llama-3.1-8B & MaxSP & GreedyNP & $-2.773$ & $7.839 \times 10^{-3}*$ \\
Llama-3.1-8B & MaxSP & MPP & $-29.933$ & $3.782 \times 10^{-33}*$ \\
Llama-3.1-8B & Random & GreedyNP & $-17.826$ & $4.641 \times 10^{-23}*$ \\
Llama-3.1-8B & Random & MPP & $-38.155$ & $4.200 \times 10^{-38}*$ \\
Llama-3.1-8B & GreedyNP & MPP & $-29.146$ & $1.304 \times 10^{-32}*$ \\
\hline
Llama-3.1-8B-Instruct & MaxSP & Random & $9.252$ & $2.491 \times 10^{-12}*$ \\
Llama-3.1-8B-Instruct & MaxSP & GreedyNP & $-5.496$ & $1.385 \times 10^{-6}*$ \\
Llama-3.1-8B-Instruct & MaxSP & MPP & $-34.577$ & $4.414 \times 10^{-36}*$ \\
Llama-3.1-8B-Instruct & Random & GreedyNP & $-10.204$ & $1.024 \times 10^{-13}*$ \\
Llama-3.1-8B-Instruct & Random & MPP & $-38.033$ & $4.884 \times 10^{-38}*$ \\
Llama-3.1-8B-Instruct & GreedyNP & MPP & $-33.829$ & $1.235 \times 10^{-35}*$ \\
\hline
\bottomrule
\end{tabular}
\end{table}

\subsection{Asymmetric Collaboration}
\label{subsec:appndx:asymmetric}
Next, we analyze the asymmetric collaboration setting. Table~\ref{appendix:tab:perp_strategy} reports Player~G's utility, measured by perplexity.

As in the full collaboration setting, both \textbf{G-Utility} and \textbf{G-Greedy} outperform the random baseline across all evaluated LLMs. Player~G achieves the highest utility when using the learned \textbf{G-Utility} strategy across all LLMs. This improvement is statistically significant for all LLMs.

Table~\ref{appendix:tab:views_strategy} reports the significance analysis for Player~F's utility, measured by normalized view counts. As in previous settings,\textbf{G-Utility} outperforms the random baseline across all evaluated LLMs, while \textbf{G-Greedy} yields inferior performance compared to the random baseline.This behavior aligns with our expectations, as discussed in Subsection~\ref{subsec: findings sequential}. Notably, \textbf{G-Utility} achieves significantly higher utility than all other strategies for every LLM. 

Overall, our findings reinforce the importance of Player~G's incentive-aware optimization in asymmetric interactions. They further suggest that Player~F may benefit from revealing aspects of its utility function, so as to encourage Player~G to select questions that yield higher utility for Player~F.

\begin{table*}[t]
\footnotesize
\caption{Significance analysis of Player~G's utility (perplexity) under asymmetric collaboration. The table reports pairwise comparisons between strategies using a weekly mean--based t-test over 52 weeks, for each evaluated LLM. The paired t-statistic are shown for each method, together with the corresponding p-values. Asterisks denote statistical significance at the $0.01$ level. Both \textbf{G-Utility} and \textbf{G-Greedy} significantly outperform the random baseline across all LLMs, while \textbf{G-Utility} achieves significantly higher utility than \textbf{G-Greedy}.}
\label{appendix:tab:perp_strategy}
\centering
\setlength{\tabcolsep}{8pt}
\begin{tabular}{l l l r l}
\textbf{LLM} & \textbf{Strategy A} & \textbf{Strategy B} & \textbf{t-statistic} & \textbf{p-value} \\ \hline
Pythia 6.9B & G-Greedy & G-Utility & $-87.766$ & $1.497 \times 10^{-55}*$ \\
Pythia 6.9B & G-Greedy & Random & $178.819$ & $1.203 \times 10^{-70}*$ \\
Pythia 6.9B & G-Utility & Random & $223.304$ & $2.281 \times 10^{-75}*$ \\
\hline
Llama-3.1-8B & G-Greedy & G-Utility & $-82.077$ & $3.908 \times 10^{-54}*$ \\
Llama-3.1-8B & G-Greedy & Random & $143.270$ & $6.137 \times 10^{-66}*$ \\
Llama-3.1-8B & G-Utility & Random & $186.432$ & $1.564 \times 10^{-71}*$ \\
\hline
Llama-3.1-8B-Instruct & G-Greedy & G-Utility & $-84.881$ & $7.620 \times 10^{-55}*$ \\
Llama-3.1-8B-Instruct & G-Greedy & Random & $162.932$ & $1.140 \times 10^{-68}*$ \\
Llama-3.1-8B-Instruct & G-Utility & Random & $185.820$ & $1.837 \times 10^{-71}*$ \\
\hline
\bottomrule
\end{tabular}
\end{table*}

\begin{table*}[t!]
\footnotesize
\caption{Significance analysis of Player~F's utility (normalized view counts) under asymmetric collaboration. The table reports pairwise comparisons between strategies using a weekly mean--based t-test over 52 weeks, for each evaluated LLM.  The paired t-statistic are reported for each method, together with the corresponding p-values. Asterisks denote statistical significance at the $0.01$ level. The learned \textbf{G-Utility} strategy significantly outperforms both \textbf{G-Greedy} and the random baseline across all LLMs, while the random baseline outperforms \textbf{G-Greedy}}

\label{appendix:tab:views_strategy}
\centering
\begin{tabular}{l l l r l}
\toprule
\textbf{LLM} & \textbf{Strategy A} & \textbf{Strategy B} & \textbf{t-statistic} & \textbf{p-value} \\ \hline
Pythia 6.9B & G-Greedy & G-Utility & $-66.461$ & $1.105 \times 10^{-49}*$ \\
Pythia 6.9B & G-Greedy & Random & $-7.488$ & $1.168 \times 10^{-9}*$ \\
Pythia 6.9B & G-Utility & Random & $37.892$ & $5.827 \times 10^{-38}*$ \\
\hline
Llama-3.1-8B & G-Greedy & G-Utility & $-73.354$ & $9.207 \times 10^{-52}*$ \\
Llama-3.1-8B & G-Greedy & Random & $-13.814$ & $1.603 \times 10^{-18}*$ \\
Llama-3.1-8B & G-Utility & Random & $31.189$ & $5.571 \times 10^{-34}*$ \\
\hline
Llama-3.1-8B-Instruct & G-Greedy & G-Utility & $-62.733$ & $1.811 \times 10^{-48}*$ \\
Llama-3.1-8B-Instruct & G-Greedy & Random & $-15.964$ & $4.738 \times 10^{-21}*$ \\
Llama-3.1-8B-Instruct & G-Utility & Random & $24.960$ & $1.620 \times 10^{-29}*$ \\
\hline
\bottomrule
\end{tabular}

\end{table*}

\fi%

\end{document}